\documentclass{article}

\usepackage{microtype}
\usepackage{graphicx}
\usepackage{subfigure}
\usepackage{booktabs} 
\usepackage{amsmath}
\usepackage{amssymb}
\usepackage{amsthm}
\usepackage{amsfonts}       
\usepackage{graphicx}
\usepackage{thm-restate}
\usepackage{algorithm}
\usepackage{algorithmic}
\usepackage{tabularx}

\usepackage{hyperref}

\newtheorem{theorem}{Theorem}

\usepackage[accepted]{icml2019}

\icmltitlerunning{Towards Characterizing Divergence in Deep Q-Learning}

\newcommand{\underE}[2]{\underset{\begin{subarray}{c}#1 \end{subarray}}{\E}\left[ #2 \right]}

\begin{document}
%
%



\newcommand{\avet}{{\mathbf  a}}
\newcommand{\bvet}{{\mathbf  b}}
\newcommand{\cvet}{{\mathbf  c}}
\newcommand{\dvet}{{\mathbf  d}}
\newcommand{\evet}{{\mathbf  e}}
\newcommand{\fvet}{{\mathbf  f}}
\newcommand{\gvet}{{\mathbf  g}}
\newcommand{\hvet}{{\mathbf  h}}
\newcommand{\ivet}{{\mathbf  i}}
\newcommand{\jvet}{{\mathbf  j}}
\newcommand{\kvet}{{\mathbf  k}}
\newcommand{\lvet}{{\mathbf  l}}
\newcommand{\mvet}{{\mathbf  m}}
\newcommand{\nvet}{{\mathbf  n}}
\newcommand{\ovet}{{\mathbf  o}}
\newcommand{\pvet}{{\mathbf  p}}
\newcommand{\qvet}{{\mathbf  q}}
\newcommand{\rvet}{{\mathbf  r}}
\newcommand{\svet}{{\mathbf  s}}
\newcommand{\tvet}{{\mathbf  t}}
\newcommand{\uvet}{{\mathbf  u}}
\newcommand{\vvet}{{\mathbf  v}}
\newcommand{\xvet}{{\mathbf  x}}
\newcommand{\yvet}{{\mathbf  y}}
\newcommand{\zvet}{{\mathbf  z}}
\newcommand{\wvet}{{\mathbf  w}}

\newcommand{\Avet}{{\mathbf  A}}
\newcommand{\Bvet}{{\mathbf  B}}
\newcommand{\Cvet}{{\mathbf  C}}
\newcommand{\Dvet}{{\mathbf  D}}
\newcommand{\Evet}{{\mathbf  E}}
\newcommand{\Fvet}{{\mathbf  F}}
\newcommand{\Gvet}{{\mathbf  G}}
\newcommand{\Hvet}{{\mathbf  H}}
\newcommand{\Ivet}{{\mathbf  I}}
\newcommand{\Jvet}{{\mathbf  J}}
\newcommand{\Kvet}{{\mathbf  K}}
\newcommand{\Lvet}{{\mathbf  L}}
\newcommand{\Mvet}{{\mathbf  M}}
\newcommand{\Nvet}{{\mathbf  N}}
\newcommand{\Ovet}{{\mathbf  O}}
\newcommand{\Pvet}{{\mathbf  P}}
\newcommand{\Qvet}{{\mathbf  Q}}
\newcommand{\Rvet}{{\mathbf  R}}
\newcommand{\Svet}{{\mathbf  S}}
\newcommand{\Tvet}{{\mathbf  T}}
\newcommand{\Uvet}{{\mathbf  U}}
\newcommand{\Xvet}{{\mathbf  X}}
\newcommand{\Yvet}{{\mathbf  Y}}
\newcommand{\Vvet}{{\mathbf  V}}
\newcommand{\Wvet}{{\mathbf  W}}
\newcommand{\Zvet}{{\mathbf  Z}}

\newcommand{\Deltavet}{\mathbf  \Delta}
\newcommand{\Lambdavet}{{\mathbf  \Lambda}}
\newcommand{\Sigmavet}{\mathbf  \Sigma}
\newcommand{\Thetavet}{{\mathbf  \Theta}}

\newcommand{\s}{ {\sigma} }

\newcommand{\e}{{\mathrm e}}
\newcommand{\jm}{{\mathrm j}}
\newcommand{\E}{{\mathrm E}}
\newcommand{\Ex}{{\mathbb E}}
\renewcommand{\d}{{\mathrm d}}
\newcommand{\dt}{{\mathrm d}t}
\newcommand{\X}{ {\mathcal X} }
\newcommand{\Y}{ {\mathcal Y} }
\newcommand{\Z}{ {\mathcal Z} }

\newcommand{\calA}{{\mathcal A}}
\newcommand{\calB}{{\mathcal B}}
\newcommand{\calC}{{\mathcal C}}
\newcommand{\calD}{{\mathcal D}}
\newcommand{\calE}{{\mathcal E}}
\newcommand{\calF}{{\mathcal F}}
\newcommand{\calG}{{\mathcal G}}
\newcommand{\calH}{{\mathcal H}}
\newcommand{\calI}{{\mathcal I}}
\newcommand{\calJ}{{\mathcal J}}
\newcommand{\calK}{{\mathcal K}}
\newcommand{\calL}{{\mathcal L}}
\newcommand{\calM}{{\mathcal M}}
\newcommand{\calN}{{\mathcal N}}
\newcommand{\calO}{{\mathcal O}}
\newcommand{\calP}{{\mathcal P}}
\newcommand{\calQ}{{\mathcal Q}}
\newcommand{\calR}{{\mathcal R}}
\newcommand{\calS}{{\mathcal S}}
\newcommand{\calT}{{\mathcal T}}
\newcommand{\calU}{{\mathcal U}}
\newcommand{\calV}{{\mathcal V}}
\newcommand{\calX}{{\mathcal X}}
\newcommand{\calY}{{\mathcal Y}}
\newcommand{\calW}{{\mathcal W}}
\newcommand{\calZ}{{\mathcal Z}}
\newcommand{\qtil}{{\tilde{q}}}
\newcommand{\td}{{\tilde{\delta}}}

\newcommand{\vect}[1]{ {\mbox{\rm vec}(#1)} }


\newcommand{\Atil}{\tilde{A}}
\newcommand{\Zhat}{\hat{Z}}
\newcommand{\Hbar}{\bar{H}}
\newcommand{\Dhat}{\hat{D}}
\newcommand{\dhat}{\hat{d}}

\newcommand{\rhat}{\hat{r}}
\newcommand{\xhat}{\hat{x}}
\newcommand{\yhat}{\hat{y}}
\newcommand{\zhat}{\hat{z}}
\newcommand{\xbar}{\bar{x}}
\newcommand{\ubar}{\bar{u}}
\newcommand{\ybar}{\bar{y}}
\newcommand{\zbar}{\bar{z}}
\newcommand{\pdot}{\dot{p}}
\newcommand{\pddot}{\ddot{p}}
\newcommand{\pbar}{\bar{p}}
\newcommand{\qdot}{\dot{q}}
\newcommand{\qddot}{\ddot{q}}
\newcommand{\qbar}{\bar{q}}
\newcommand{\xdot}{\dot{x}}
\newcommand{\ydot}{\dot{y}}
\newcommand{\zdot}{\dot{z}}
\newcommand{\yddot}{\ddot{y}}
\newcommand{\thdot}{\dot{\theta}}
\newcommand{\thddot}{\ddot{\theta}}
\newcommand{\util}{{\tilde{u}}}
\newcommand{\xtil}{{\tilde{x}}}
\newcommand{\ytil}{{\tilde{y}}}
\newcommand{\lam}{\lambda}
\newcommand{\lamax}{\lambda\ped{max}}
\newcommand{\lamin}{\lambda\ped{min}}
\newcommand{\adj}{ {\mbox{\rm adj}\;} }
\newcommand{\sign}{\mbox {\rm sgn}}
\newcommand{\spn}{\mbox {\rm span}}
\newcommand{\barJ}{\bar{J}}
\newcommand{\dom}{\mathop {\mathrm {dom}}}
\newcommand{\card}{\mathop{\mathrm{card}}}
\newcommand{\subt}{\mathop{\mathrm{s.t.}}}

\newcommand{\epi}{\mathop{\mathrm{epi}}}
\newcommand{\env}{\mathop{\mathrm{env}}}
\newcommand{\chull}{\mathop{\mathrm{co}}}
\newcommand{\graph}{\mathop{\mathrm{graph}}}
\newcommand{\prox}[1]{\mathop{\mathrm{prox}_{#1}}}
\newcommand{\sthr}[1]{\mathop{\mathrm{sthr}_{#1}}}

\def\hardsection{$\spadesuit\;$}

\newcommand{\Real}[1]{ { {\mathbb R}^{#1} } }
\newcommand{\Realp}[1]{ { {\mathbb R}_{+}^{#1} } }
\newcommand{\Realpp}[1]{ { {\mathbb R}_{++}^{#1} } }
\newcommand{\Complex}[1]{ { {\mathbb C}^{#1} } }
\newcommand{\Imag}[1]{ { {\mathbb I}^{#1} } }
\newcommand{\Field}[1]{ {\mathbb F}^{#1} }
\newcommand{\F}{ {\mathbb F}}
\newcommand{\Orth}[1]{ { {\calG_{\calO}^{#1}} } }
\newcommand{\Unit}[1]{ { {\calG_{\calU}^{#1}} } }
\newcommand{\Sym}[1]{ { {\mathbb S}^{#1} } }
\newcommand{\Symp}[1]{ { {\mathbb S}_{+}^{#1} } }
\newcommand{\Sympp}[1]{ { {\mathbb S}_{++}^{#1} } }
\newcommand{\Herm}[1]{ { {\mathbb H}^{#1} } }
\newcommand{\Skew}[1]{ { {\mathbb S\mathbb K}^{#1} } }
\newcommand{\Skherm}[1]{ { {\mathbb H\mathbb K}^{#1} } }
\newcommand{\Rman}[1]{ { {\mathcal R}^{#1} } } 
\newcommand{\Cman}[1]{ { {\mathcal C}^{#1} } }
\newcommand{\Hinf}[1]{ {  {\mathcal H}_\infty^{#1} } }
\newcommand{\RHinf}[1]{ { {\mathcal RH}_\infty^{#1} } }
\newcommand{\Htwo}[1]{ {  {\mathcal H}_2^{#1} } }
\newcommand{\RHtwo}[1]{ { {\mathcal RH}_2^{#1} } }

\newcommand{\dist}[1]{{\mathrm{dist}}{\left( #1 \right)}}
\newcommand{\diff}[2]{ \frac{\d {#1}}{\d {#2}}  }
\newcommand{\diffp}[2]{ \frac{\partial {#1}}{\partial {#2}}  }
\newcommand{\diffqd}[2]{ \frac{\d^2 {#1}}{\d {#2}^2}  }
\newcommand{\diffq}[2]{ \frac{\d^2 {#1}}{\d {#2}}  }
\newcommand{\diffqq}[3]{ \frac{\d^2 {#1}}{ \d {#2} \d {#3}  }}
\newcommand{\diffpq}[2]{ \frac{\partial^2 {#1}}{\partial {#2}^2}  }
\newcommand{\difftq}[3]{ \frac{\partial^2 {#1}}{\partial {#2}\partial {#3}}  }
\newcommand{\diffi}[3]{ \frac{\d^{#3} {#1}}{\d {#2}^{#3}}  }
\newcommand{\diffpi}[3]{ \frac{\partial^{#3} {#1}}{\partial {#2}^{#3}}  }
\newcommand{\binomial}[2]{\scriptsize{\left(\!\! \ba{c} #1 \\ #2 \ea \!\! \right)} }
\newcommand{\comb}[2]{{\left(\!\!\! \ba{c} #1 \\ #2 \ea \!\!\! \right)} }

\newcommand{\simax}{{\sigma_{\mathrm{max}}}}
\newcommand{\simin}{{\sigma_{\mathrm{min}}}}
\newcommand{\prob}{{\mbox{\rm Prob}}}
\newcommand{\var}{{\mbox{\rm var}}}
\newcommand{\sint}{{\mbox{\rm int}\,}} 
\newcommand{\relint}{{\mbox{\rm relint}\,}} 
\newcommand{\ns}{{\mbox{\tt ns}}}

\newcommand{\rank}{\mathop{\mathrm{rank}}\nolimits}
\newcommand{\range}{\mathop{\mathcal{R}}\nolimits}
\newcommand{\nulsp}{\mathop{\mathcal{N}}\nolimits}
\newcommand{\diagop}{\mathop{\mathrm{diag}}\nolimits}
\newcommand{\Var}{\mathop{\mathrm{var}}\nolimits}
\newcommand{\tr}{\mathop{\mathrm{trace}}\nolimits}
\newcommand{\sinc}{\mathop{\mathrm{sinc}}\nolimits}

\newcommand{\pre}[1]{ { {\mathop{\mathrm{Re}}}  \left({#1}\right)} }
\newcommand{\pim}[1]{ { {\mathop{\mathrm{Im}}}  ({#1})} }
\newcommand{\rp}{ ^{\Real{}} }
\newcommand{\ip}{ ^{\Imag{}} }

\newcommand{\one}{{\mathbf  1}}
\newcommand{\dss}{\displaystyle}
\newcommand{\inv}{^{-1}}
\newcommand{\pinv}{^{\dagger}}
\newcommand{\diag}[1]{\mathrm{diag}\left({#1}\right)}
\newcommand{\blockdiag}[1]{\mbox{\rm bdiag}\left({#1}\right)}
\newcommand{\tran}{^{\top}}
\newcommand{\inner}[1]{\langle {#1} \rangle}
\newcommand{\ped}[1]{_{\mathrm{#1}}}
\newcommand{\ap}[1]{^{\mathrm{#1}}}

\newcommand{\blu}[1]{\textcolor{blue}{#1}}
\newcommand{\red}[1]{\textcolor{red}{#1}}
\newcommand{\green}[1]{\textcolor{green}{#1}}
\newcommand{\cyan}[1]{\textcolor{cyan}{#1}}

\newcommand{\beq}{\begin{equation}}
\newcommand{\eeq}{\end{equation}}
\newcommand{\bea}{\begin{eqnarray}}
\newcommand{\eea}{\end{eqnarray}}
\newcommand{\beas}{\begin{eqnarray*}}
\newcommand{\eeas}{\end{eqnarray*}}
\newcommand{\ba}{\begin{array}}
\newcommand{\ea}{\end{array}}
\newcommand{\bit}{\begin{itemize}}
\newcommand{\eit}{\end{itemize}}
\newcommand{\ben}{\begin{enumerate}}
\newcommand{\een}{\end{enumerate}}
\newcommand{\bde}{\begin{description}}
\newcommand{\ede}{\end{description}}
\newcommand{\bsp}{\begin{split}}
\newcommand{\esp}{\end{split}}


%
%

\def\nocolon{}

\newcommand{\monthyear}{%
  \ifcase\month\or January\or February\or March\or April\or May\or June\or
  July\or August\or September\or October\or November\or
  December\fi\space\number\year
}

\newcommand{\openepigraph}[2]{%
  \begin{fullwidth}
  \sffamily\large
  \begin{doublespace}
  \noindent\allcaps{#1}\\
  \noindent\allcaps{#2}
  \end{doublespace}
  \end{fullwidth}
}

\newcommand{\blankpage}{\newpage\hbox{}\thispagestyle{empty}\newpage}

\twocolumn[
\icmltitle{Towards Characterizing Divergence in Deep Q-Learning}




\begin{icmlauthorlist}
\icmlauthor{Joshua Achiam}{openai,berk}
\icmlauthor{Ethan Knight}{openai,nueva}
\icmlauthor{Pieter Abbeel}{berk,covariant}
\end{icmlauthorlist}

\icmlaffiliation{openai}{OpenAI}
\icmlaffiliation{berk}{UC Berkeley}
\icmlaffiliation{nueva}{The Nueva School}
\icmlaffiliation{covariant}{Covariant.AI}

\icmlcorrespondingauthor{Joshua Achiam}{jachiam@openai.com}

\icmlkeywords{Machine Learning, ICML}

\vskip 0.3in
]



\printAffiliationsAndNotice{}  

\begin{abstract}
    Deep Q-Learning (DQL), a family of temporal difference algorithms for control, employs three techniques collectively known as the `deadly triad' in reinforcement learning: bootstrapping, off-policy learning, and function approximation. Prior work has demonstrated that together these can lead to divergence in Q-learning algorithms, but the conditions under which divergence occurs are not well-understood. In this note, we give a simple analysis based on a linear approximation to the Q-value updates, which we believe provides insight into divergence under the deadly triad. The central point in our analysis is to consider when the leading order approximation to the deep-Q update is or is not a contraction in the sup norm. Based on this analysis, we develop an algorithm which permits stable deep Q-learning for continuous control without any of the tricks conventionally used (such as target networks, adaptive gradient optimizers, or using multiple Q functions). We demonstrate that our algorithm performs above or near state-of-the-art on standard MuJoCo benchmarks from the OpenAI Gym.
\end{abstract}

\section{Introduction}

Deep Q-Learning (DQL), a family of reinforcement learning algorithms that includes Deep Q-Network (DQN) \cite{Mnih2013, Mnih2015} and its continuous-action variants \cite{Lillicrap2016, Fujimoto2018, Haarnoja2018}, is often successful at training deep neural networks for control. In DQL, a function approximator (a deep neural network) learns to estimate the value of each state-action pair under the optimal policy (the $Q$-function), and a control policy selects actions with the highest values according to the current $Q$-function approximator. DQL algorithms have been applied fruitfully in video games \cite{Mnih2015}, robotics \cite{Kalashnikov2018, Haarnoja2018a}, and user interactions on social media \cite{Gauci2018}. 

However, despite the high-profile empirical successes, these algorithms possess failure modes that are poorly understood and arise frequently in practice. The most common failure mode is divergence, where the $Q$-function approximator learns to ascribe unrealistically high values to state-action pairs, in turn destroying the quality of the greedy control policy derived from $Q$ \cite{VanHasselt2018}. Divergence in DQL is often attributed to three components common to all DQL algorithms, which are collectively considered the `deadly triad' of reinforcement learning \cite{Sutton1988, Sutton2018}: 
\begin{itemize}
\item function approximation, in this case the use of deep neural networks,
\item off-policy learning, the use of data collected on one policy to estimate the value of another policy,
\item and bootstrapping, where the $Q$-function estimator is regressed towards a function of itself. 
\end{itemize} 
Well-known examples \cite{Baird1995, Tsitsiklis1997} demonstrate the potential of the deadly triad to cause divergence in approximate $Q$-learning. However, actionable descriptions of divergence in the general case remain elusive, prompting algorithm designers to attack the triad with an increasingly wide variety of heuristic solutions. These include target networks \cite{Mnih2015}, entropy regularization \cite{Fox2016, Haarnoja2018}, n-step learning \cite{Hessel2017}, and approximate double-Q learning \cite{VanHasselt2016}. 

The absence of theoretical characterization for divergence in DQL makes it challenging to reliably deploy DQL on new problems. To make progress toward such a characterization, we give an analysis inspired by \citet{Gordon1995}, who studied the behavior of approximate value learning algorithms in value space. We examine how $Q$-values change under a standard DQL update, and derive the leading order approximation to the DQL update operator. The approximate update turns out to have a simple form that disentangles and clarifies the contributions from the components of the deadly triad, and allows us to identify the important role played by the neural tangent kernel (NTK) \cite{Jacot2018} of the $Q$ approximator. We consider conditions under which the approximate update is or isn't a contraction map in the sup norm, based on the intuition that when it is a contraction DQL should behave stably, and when it is an expansion we should expect divergence. 

Based on our analysis, we design an algorithm which is intended to approximately ensure that the $Q$-function update is non-expansive. Our algorithm, which we call Preconditioned Q-Networks (PreQN) is computationally expensive but theoretically simple: it works by preconditioning the TD-errors in minibatch gradient updates, using the inverse of a matrix of inner products of $Q$-function gradients. We demonstrate that PreQN is stable and performant on a standard slate of MuJoCo benchmarks from the OpenAI Gym \cite{Brockman2016}, despite using none of the tricks typically associated with DQL. We also find a neat connection between PreQN and natural gradient \cite{Amari1998} methods, where under some slightly restricted conditions, the PreQN update is equivalent to a natural gradient Q-learning update. This connection explains a result noted by \citet{Knight2018}: that natural gradient Q-learning appeared to be stable without target networks.


\section{Preliminaries}

\subsection{Contraction Maps and Fixed Points}

We begin with a brief mathematical review. Let $X$ be a vector space with norm $\|\cdot\|$, and $f$ a function from $X$ to $X$. If $\forall x, y \in X$, $f$ satisfies
\begin{equation}
\|f(x) - f(y)\| \leq \beta \|x - y\| \label{contraction}
\end{equation}
with $\beta \in [0,1)$, then $f$ is called a contraction map with modulus $\beta$. If $f$ satisfies Eq \ref{contraction} but with $\beta=1$, then $f$ is said to be a non-expansion. 

By the Banach fixed-point theorem, if $f$ is a contraction, there is a unique fixed-point $x$ such that $f(x)=x$, and it can be obtained by the repeated application of $f$: for any point $x_0 \in X$, if we define a sequence of points $\{x_n\}$ such that $x_n = f(x_{n-1})$, $\lim_{n\to\infty} x_n = x$.

\subsection{Q Functions and TD-Learning}
DQL algorithms learn control policies in the reinforcement learning (RL) setting with the infinite-horizon discounted return objective. They attempt to learn an approximator to the optimal action-value function $Q^*$, which is known to satisfy the optimal Bellman equation:
\begin{equation}
Q^*(s,a) = \underE{s' \sim P}{R(s,a,s') + \gamma \max_{a'} Q^* (s', a')}. \label{bellman}
\end{equation}
Here, $s$ and $s'$ are states, $a$ and $a'$ are actions, $P$ is the transition kernel for the environment, $R$ is the reward function, and $\gamma \in (0,1)$ is the discount factor. The optimal policy $\pi^*$ can be obtained as the policy which is greedy with respect to $Q^*$: $\pi^*(s) = \arg \max_{a} Q^*(s,a)$. Thus, DQL algorithms approximate the optimal policy as the policy which is greedy with respect to the $Q^*$-approximator.

Let $\calT^* : \calQ \to \calQ$ be the operator on Q functions with $\calT^*Q (s,a)$ given by the RHS of Eq \ref{bellman}; then Eq \ref{bellman} can be written as $Q^* = \calT^* Q^*$.  The operator $\calT^*$ is called the optimal Bellman operator, and it is a contraction in the sup norm with modulus $\gamma$. When the $Q$-function can be represented as a finite table and the reward function and transition kernel are fully known, $\calT^*$ can be computed exactly and $Q^*$ can be obtained by \textit{value iteration}: $Q_{k+1} = \calT^*Q_k$. The convergence of value iteration from any initial point $Q_0$ is guaranteed by the Banach fixed-point theorem.

When the reward function and transition kernel are not fully known, it is still possible to learn $Q^*$ in the tabular setting via $Q$-learning \cite{Watkins1992}. $Q$-learning updates the $Q$-values of state-action pairs as they are visited by some exploration policy according to:
\begin{equation}
Q_{k+1}(s,a) = Q_k (s,a) + \alpha_k \left( \hat{\calT}^*Q_k (s,a) - Q_k(s,a)\right), \label{watkins}
\end{equation}
where $\hat{\calT}^* Q_k(s,a) = r + \gamma \max_{a'} Q_k(s',a')$ is a sample estimate for $\calT^* Q_k (s,a)$ using the reward and next state obtained from the environment while exploring. Under mild conditions on state-action visitation (all pairs must be visited often enough) and learning rates $\alpha_k$ (they must all gradually go to zero and always lie in $[0,1)$), $Q$-learning converges to $Q^*$. $Q$-learning is called a \textit{temporal difference} (TD) algorithm because the updates to $Q$-values are based on the temporal difference error:
\begin{align*}
\delta_t &= \hat{\calT}^* Q(s_t,a_t) - Q(s_t,a_t) \\
&= r_t + \gamma \max_{a'} Q(s_{t+1},a') - Q(s_t,a_t).
\end{align*}

\section{Towards Characterizing Divergence in Deep Q-Learning}

Deep Q-Learning (DQL) algorithms are based on the generalization of Eq \ref{watkins} to the function approximation setting:
\begin{equation}
\theta' = \theta + \alpha \left(\hat{\calT}^*Q_{\theta}(s,a) - Q_{\theta}(s,a)\right) \nabla_{\theta}Q_{\theta}(s,a), \label{approxqlearning}
\end{equation}
where $Q_{\theta}$ is a differentiable function approximator with parameters $\theta$, and $\theta'$ are the parameters after an update. Note that when $Q_{\theta}$ is a table, Eq \ref{approxqlearning} reduces exactly to Eq \ref{watkins}.

Typically, DQL algorithms make use of experience replay \cite{Lin1992} and minibatch gradient descent \cite{Mnih2013}, resulting in an expected update:
\begin{equation}
\theta' = \theta + \alpha \underE{s,a \sim \rho}{ \left(\calT^*Q_{\theta}(s,a) - Q_{\theta}(s,a)\right) \nabla_{\theta}Q_{\theta}(s,a)}, \label{deepq}
\end{equation}
where $\rho$ is the distribution of experience in the replay buffer at the time of the update. For stability, it is conventional to replace the bootstrap, the $\calT^* Q_{\theta}$ term, with one based on a slowly-updated target network: $\calT^*Q_{\psi}$, where the parameters $\psi$ are either infrequently copied from $\theta$ \cite{Mnih2013} or obtained by Polyak averaging $\theta$ \cite{Lillicrap2016}. However, we will omit target networks from our analysis.

Unlike Q-learning, DQL in its standard form currently has no known convergence guarantees, although some convergence results \cite{Yang2019} have been obtained for a closely-related algorithm called Neural-Fitted Q Iteration \cite{Riedmiller2005} when used with deep ReLU networks. 

\subsection{Taylor Expansion Analysis of Q}

To gain a deeper understanding of the behavior of DQL, we study the change in Q-values following an update based on Eq \ref{deepq}, by examining the Taylor expansion of $Q$ around $\theta$ at a state-action pair $(\bar{s},\bar{a})$. The Taylor expansion is
\begin{equation}
Q_{\theta'}(\bar{s},\bar{a}) = Q_{\theta}(\bar{s}, \bar{a}) + \nabla_{\theta} Q_{\theta}(\bar{s},\bar{a})^T (\theta' - \theta) + \calO\left(\|\theta'-\theta\|^2\right), \label{taylor}
\end{equation}
and by plugging Eq \ref{deepq} into Eq \ref{taylor}, we obtain:
\begin{align}
Q_{\theta'}(\bar{s},&\bar{a}) = \\
& Q_{\theta}(\bar{s}, \bar{a}) \nonumber \\
& + \alpha \underE{s,a\sim\rho}{k_{\theta}(\bar{s},\bar{a},s,a) \left(\calT^*Q_{\theta}(s,a) - Q_{\theta}(s,a)\right)} \nonumber \\
& + \calO\left(\|\alpha g\|^2\right), \label{deepqtaylor}
\end{align}
where  
\begin{equation}
k_{\theta}(\bar{s},\bar{a},s,a)\doteq \nabla_{\theta} Q_{\theta}(\bar{s}, \bar{a})^T \nabla_{\theta} Q_{\theta}(s,a) \label{kernelcomponent}
\end{equation}
is the neural tangent kernel (NTK) \cite{Jacot2018}, and $\alpha g = \theta' - \theta$. It is instructive to look at Eq \ref{deepqtaylor} for the case of finite state-action spaces, where we can consider its matrix-vector form. 

\begin{theorem} For Q-learning with nonlinear function approximation based on the update in Eq \ref{deepq}, when the state-action space is finite and the $Q$ function is represented as a vector in $\Real{|S||A|}$, the $Q$-values before and after an update are related by:
\begin{equation}
Q_{\theta'} = Q_{\theta} + \alpha K_{\theta} D_{\rho} \left(\calT^*Q_{\theta} - Q_{\theta} \right) + \calO(\|\alpha g\|^2), \label{deadlytriad}
\end{equation}
where $K_{\theta}$ is the $|S||A|\times|S||A|$ matrix of entries given by Eq \ref{kernelcomponent}, and $D_{\rho}$ is a diagonal matrix with entries given by $\rho(s,a)$, the distribution from the replay buffer.
\end{theorem}

Although extremely simple, we believe that Eq \ref{deadlytriad} is illuminating because it shows the connection between the deadly triad and the thing we really care about: the $Q$-values themselves. At leading order,
\begin{itemize}
\item the $K_{\theta}$ term is the contribution from function approximation, with its off-diagonal terms creating generalization across state-action pairs,
\item the $D_{\rho}$ term is the contribution from the off-policy data distribution,
\item the $\calT^*Q_{\theta}$ term is the contribution from bootstrapping,
\end{itemize}
and these terms interact by multiplication. The form of Eq \ref{deadlytriad} suggests that a useful way to think about the stability and convergence of deep Q-learning is to reason about whether the leading-order update operator $\calU : \calQ \to \calQ$ with values
\begin{equation}
\calU Q_{\theta} = Q_{\theta} + \alpha K_{\theta} D_{\rho} \left(\calT^*Q_{\theta} - Q_{\theta} \right) \label{dqlupdate}
\end{equation}
is or is not a contraction on $\calQ$. In what follows, we will develop an intuition for such conditions by considering a sequence of operators that incrementally introduce the components of $\calU$. After building intuition for failure modes, we will consider how prior methods try to repair or mitigate them. We contend that prior work in DQL predominantly focuses on the data distribution or the bootstrap, with limited exploration of the contribution from $K_{\theta}$ to instability. This analysis inspires PreQN, our new algorithm which attempts to repair divergence issues by cancelling out within-batch generalization errors created by $K_{\theta}$.

\subsection{Building Intuition for Divergence}

The aim of this section is to understand how the update $\calU : \calQ \to \calQ$ given by Eq \ref{dqlupdate} can give rise to instability in deep Q-learning, and how that instability might be repaired. To begin with, consider the operator $\calU_1$ given by
\begin{equation}
\calU_1 Q = Q + \alpha \left(\calT^*Q - Q\right). \label{update1}
\end{equation}

\begin{restatable}{lemma}{lemuone} For $\alpha \in (0,1)$, $\calU_1$ given by Eq \ref{update1} is a contraction on $\calQ$ in the sup norm, and its fixed-point is $Q^*$.
\end{restatable}
Proof for this and all other results in appendix. With sampling, $\calU_1$ would be essentially the same operator as used in tabular Q-learning \cite{Watkins1992}, and it would benefit from similar performance guarantees. This gives us intuition point 1:

\textbf{Intuition 1:} When $\calU$ more closely resembles $\calU_1$, we should expect deep Q-learning to behave more stably.

Next, we consider the operator $\calU_2$ given by
\begin{equation}
\calU_2 Q = Q + \alpha D_{\rho}\left(\calT^*Q - Q\right), \label{update2}
\end{equation}
where $D_{\rho}$ is a diagonal matrix with entries $\rho(s,a)$, a probability mass function on state-action pairs.

\begin{restatable}{lemma}{lemutwo} If $\rho(s,a) > 0$ for all $s,a$ and $\alpha \in (0, 1/\rho_{max})$ where $\rho_{max} = \max_{s,a} \rho(s,a)$, then $\calU_2$ given by Eq \ref{update2} is a contraction in the sup norm and its fixed-point is $Q^*$. If there are any $s,a$ such that $\rho(s,a)=0$  and $\alpha \in (0,1/\rho_{max})$, however, it is a non-expansion in $Q$ and not a contraction.
\end{restatable}


By considering $\calU_2$, we see how the data distribution can have an impact: as long as the exploration policy touches all state-action pairs often enough, $\calU_2$ behaves well, but missing data poses a problem. The $Q$-values for missing state-action pairs will never change from their initial values, and bootstrapping will cause those erroneous values to influence the $Q$-values for `downstream' state-action pairs. This leads us to our second point of intuition:

\textbf{Intuition 2:} When data is scarce, deep Q-learning may struggle, and initial conditions will matter more.

Data is scarcest at the beginning of training; empirical results from \citet{VanHasselt2018} suggest that this is when DQL is most susceptible to divergence.

Next, we consider the operator $\calU_3$ given by
\begin{equation}
\calU_3 Q = Q + \alpha K D_{\rho}\left(\calT^*Q - Q\right), \label{update3}
\end{equation}
where $K$ is a constant symmetric, positive-definite matrix. Interestingly, $\calU_3$ corresponds exactly to the expected update for the case of linear function approximation, which we make precise in the next lemma:

\begin{restatable}{lemma}{lemuthree} For $Q$-learning with linear function approximators of the form  $Q_{\theta}(s,a) = \theta^T \phi(s,a)$ and updates based on Eq. \ref{deepq}, under the same conditions as Theorem 1, the $Q$-values before and after an update are related by
\begin{equation}
Q_{\theta'} = \calU_3 Q_{\theta}, \label{lindeadlytriad}
\end{equation}
where $K(\bar{s},\bar{a},s,a) = \phi(\bar{s},\bar{a})^T \phi(s,a)$. Eq. \ref{lindeadlytriad} differs from Eq. \ref{deadlytriad} in that there are no higher-order terms, and $K$ is constant with respect to $\theta$.
\end{restatable}

We now consider when $\calU_3$ is a contraction in the sup norm. 
\begin{restatable}{theorem}{thmtwo}
Let indices $i$, $j$ refer to state-action pairs. Suppose that $K$ and $\rho$ satisfy the conditions:
\begin{align}
\forall i, & \;\;\;\;\; \alpha K_{ii} \rho_i < 1, \label{overshoot}\\
\forall i, & \;\;\;\;\; (1+\gamma) \sum_{j\neq i} |K_{ij}| \rho_j \leq (1-\gamma) K_{ii}\rho_i. \label{diagdom}
\end{align}
Then $\calU_3$ is a contraction on $\calQ$ in the sup norm, with fixed-point $Q^*$.
\end{restatable}
To frame this discussion, we'll note that the condition in Eq. \ref{diagdom} is extremely restrictive. It requires that $\rho > 0$ everywhere, and that the off-diagonal terms of $K$ are very small relative to the on-diagonal terms (for typical choices of $\gamma$, eg $\gamma=0.99$). As a result, an analysis of this kind may not suffice to explain the success of linear TD-learning in typical use cases. But we nonetheless find this useful in motivating the following point of intuition:

\textbf{Intuition 3:} The stability of $Q$-learning is tied to the generalization properties of the $Q$-approximator. Approximators with more aggressive generalization (larger off-diagonal terms in $K_{\theta}$) are less likely to demonstrate stable learning.

So far, we have reasoned about several individual update operators, but we have not made explicit reference to the full dynamics of training with nonlinear function approximators. In deep Q-learning, both the kernel $K_{\theta}$ and the data distribution $\rho$ change between update steps. Thus, each step can be viewed as applying a different update operator. It is important to ask if our intuitions so far have any bearing in this setting; this is the subject of our next result.

\begin{restatable}{theorem}{thmthree} Consider a sequence of updates $\{\calU_0, \calU_1, ...\}$, with each $\calU_i : \calQ \to \calQ$ Lipschitz continuous, with Lipschitz constant $\beta_i$, with respect to a norm $\|\cdot\|$. Furthermore, suppose all $\calU_i$ share a common fixed-point, $\tilde{Q}$. Then for any initial point $Q_0$, the sequence of iterates produced by $Q_{i+1} = \calU_i Q_i$ satisfies:
\begin{equation}
\|\tilde{Q} - Q_i\| \leq \left(\prod_{k=0}^{i-1} \beta_k\right) \|\tilde{Q} - Q_0\|. \label{qconverge}
\end{equation}
Furthermore, if there is an iterate $j$ such that $\forall k \geq j, \; \beta_k \in [0,1)$, the sequence $\{Q_0, Q_1, ...\}$ converges to $\tilde{Q}$. 
\end{restatable}

Roughly speaking, this theorem says that if you sequentially apply different contraction maps with the same fixed-point, you will attain that fixed-point. 

In DQL, the common fixed-point between all update operators based on Eq. \ref{deepq} is $Q^*$. For neural network approximators commonly used in practice, such update operators are unlikely to be contractions and convergence to $Q^*$ is out of reach (especially considering that $Q^*$ may not even be expressible in the approximator class). Nonetheless, we view Theorem 3 as motivating our final point of intuition:

\textbf{Intuition 4:} Although the DQL update operator varies between steps, intuitions from the constant-update setting can provide useful guidance for understanding and repairing divergence issues in DQL.

To sum up, we enumerate and discuss the failure modes for DQL that appear likely based on our analysis so far. 

\textbf{Failure Mode 1: Linearization breaks.} The learning rate $\alpha$ is too high, second-order terms in Eq. \ref{deadlytriad} are large, and updates do not correlate with Bellman updates in any meaningful way. (Based on Theorem 1.)

\textbf{Failure Mode 2: Overshooting.} The learning rate $\alpha$ is small enough for the linearization to approximately hold, but is large enough that $\calU$ from Eq. \ref{dqlupdate} is sometimes an expansion. (Based on Theorem 2, Eq \ref{overshoot}.)

\textbf{Failure Mode 3: Over-generalization.} The kernel matrix $K_{\theta}$ has large off-diagonal terms, causing the $Q$ function to generalize too aggressively and making $\calU$ sometimes an expansion. (Based on Theorem 2, Eq \ref{diagdom}.)

\textbf{Failure Mode 4: Extrapolation error.} The data distribution used for the update is inadequate. $Q$-values for missing (or under-represented) state-action pairs are adjusted solely or primarily by generalization, which sometimes produces errors. Bootstrapping then propagates those errors through the $Q$-values for all other state-action pairs. (Based on Lemma 2.) This failure mode was previously identified, named, and studied empirically by \cite{Fujimoto2018a}.

It is important to note that these failure modes may not present in clearly-distinguishable ways: indeed, they can cascade into each other, creating feedback loops that lead to divergence. For instance, consider the interaction between over-generalization and extrapolation error. A network with limited generalization would keep the $Q$-values for missing state-action pairs close to their initial values. This would lead to inaccurate, but not necessarily divergent, downstream $Q$-values. On the other hand, a network that over-generalizes will significantly alter the $Q$-values for missing state-action pairs. A slight positive bias (where all of those $Q$-values increase) will get propagated to downstream $Q$-values due to extrapolation error, making them optimistic. But this reinforces the positive bias in the generalization to missing state-action pair $Q$-values---creating a feedback loop, and ultimately divergence.

\subsection{Interpreting Prior Work}

A substantial body of prior work on stabilizing DQL focuses on modifying either the data distribution or the TD-errors.

Data distribution-based methods include massive-scale experience collection, as in Gorila-DQN \cite{Nair2015}, Ape-X \cite{Horgan2018}, and R2D2 \cite{Kapturowski2019}, and methods for improved exploration, like entropy regularization \cite{Haarnoja2018}. We speculate that such methods improve stability in DQL primarily by mitigating extrapolation error, by reducing the number of missing state-action pairs. As an alternative to improved data collection, BCQ \cite{Fujimoto2018a} mitigates extrapolation error by simply preventing missing state-action pairs from being used to form the Bellman backup.

TD error-based methods include target networks \cite{Mnih2015}, clipped TD errors \cite{Mnih2015} (commonly implemented via the Huber loss function \cite{Sidor2017}), double DQN \cite{VanHasselt2016}, n-step backups \cite{Sutton1988, Hessel2017}, transformed Bellman operators \cite{Pohlen2018}, and clipped double-Q learning \cite{Fujimoto2018}. These methods do not directly attack specific failure modes, but we speculate that they interfere with error propagation by preventing bad updates from quickly spreading to downstream $Q$-values. This allows more time for bad updates to get averaged out, or for missing data to be collected.

Relatively little work focuses on over-generalization. Ape-X DQfD \cite{Pohlen2018} uses an auxilliary temporal consistency (TC) loss to prevent the $Q$-values of target state-action pairs from changing; ablation analysis suggested that the TC loss was critical to performance. Similarly, \cite{Durugkar2017} proposed Constrained Q-Learning, which uses a constraint to prevent the average target value from changing after an update; however, \cite{Pohlen2018} did not find evidence that this technique worked on complex problems.

To the best of our knowledge, no prior work addresses the root cause of overshooting or over-generalization failures in DQL: the neural tangent kernel, $K_{\theta}$. Work in this direction would either modify network architecture to result in a $K_{\theta}$ more favorable to stability, or modify the update rule in a way which controls the influence of $K_{\theta}$ on generalization. Dueling DQN \cite{Wang2016} does modify network architecture in a way which is known to improve training on Atari games, but there is currently no known theoretical explanation for its benefits. We speculate that an analysis based on $K_{\theta}$ may prove insightful, though we have not yet found a clear result on this despite preliminary effort. The general absence of work on DQL stability via $K_{\theta}$ is the inspiration for our algorithmic contributions.

\section{Preconditioned Q-Networks}

In this section, we will introduce Preconditioned Q-Networks (PreQN), an algorithm which is intended to approximately ensure that the $Q$-function update is a non-expansion. The core idea is to alter the DQL update so that it behaves as much as possible like Eq. \ref{update1} in $Q$-value space. 

Concretely, let $\Phi_{\theta} \in \Real{d \times |S||A|}$ denote the matrix whose columns are $\nabla_{\theta} Q_{\theta}(s,a)$. To first order, what we have is
\begin{equation}
Q_{\theta'} \approx Q_{\theta} + \Phi_{\theta}^T (\theta' - \theta), \label{have}
\end{equation}
and what we want is an update which results in
\begin{equation}
Q_{\theta'} \approx Q_{\theta} + \alpha \left( \calT^* Q_{\theta} - Q_{\theta}\right), \label{want}
\end{equation}
for some $\alpha \in (0,1)$. If $K_{\theta} = \Phi_{\theta}^T \Phi_{\theta}$ were invertible, then the update
\begin{equation}
\theta' = \theta + \alpha \Phi_{\theta} K_{\theta}^{-1} \left( \calT^* Q_{\theta} - Q_{\theta}\right) \label{theoryupdate}
\end{equation}
would attain Eq \ref{want}. This update is like a normal DQL update where the TD-errors have been replaced with \textit{preconditioned} TD-errors, where $K_{\theta}^{-1}$ is the preconditioner. In practice, there are three obstacles to directly implementing Eq \ref{theoryupdate}:
\begin{itemize}
\item For large or continuous state or action spaces (as in many problems of interest), it would be intractable or impossible to form and invert $K_{\theta}$.
\item If the number of state-action pairs is greater than the number of parameters, $K_{\theta}$ will be rank deficient and thus not invertible.
\item For nonlinear function approximators (as in DQL), step sizes must be selected to keep higher-order terms small.
\end{itemize}
To handle these issues, we propose a minibatch-based approximation to the algorithm in Eq \ref{theoryupdate}. Like in standard DQL, we maintain a replay buffer filled with past experiences. Each time we sample a minibatch $B$ from the replay buffer to compute an update, we form $K_{\theta}$ for the minibatch, find the least-squares solution $Z$ to
\begin{equation}
K_{\theta} Z = \calT^* Q_{\theta} - Q_{\theta} \label{preqn1}
\end{equation}
for the minibatch, and then compute a proposed update
\begin{equation}
\theta' = \theta + \alpha \sum_{(s,a) \in B} Z(s,a) \nabla_{\theta} Q_{\theta}(s,a). \label{preqnupdate}
\end{equation}
Finally, to ensure that the higher-order terms are small, we use a linesearch that starts at Eq \ref{preqnupdate} and backtracks (by exponential decay) to $\theta$. The acceptance criterion for the linesearch is
\begin{equation}
\cos\left( Q_{\theta'} - Q_{\theta}, \calT^*Q_{\theta} - Q_{\theta}\right) \geq \eta,
\end{equation}
where $\eta$ is a hyperparameter (close to, but less than, $1$). That is, a proposed update is only accepted if the resulting change in $Q$-values for the minibatch is well-aligned with its TD-errors. We refer to this algorithm as Preconditioned Q-Networks (PreQN). 

In our experiments, we consider the variant of PreQN styled after DDPG \cite{Lillicrap2016}, where a separate actor network is trained to allow efficient computation of $\max_{a} Q_{\theta}(s,a)$. We give the complete pseudocode as Algorithm \ref{alg1}. Note the omission of target networks: we hypothesize that the design of the algorithm makes instability less likely and thus makes target networks unnecessary.

\begin{algorithm}[t]
\small
   \caption{PreQN (in style of DDPG)}
   \label{alg1}
\begin{algorithmic}[1]
     \STATE Given: initial parameters $\theta, \phi$ for $Q, \mu$, empty replay buffer $\calD$
	 \STATE Receive observation $s_0$ from environment
	 \FOR{$t = 0,1,2,...$} 
	 \STATE Select action $a_t = \mu_{\phi}(s_t) + \calN_t$
	 \STATE Step environment to get $s_{t+1}, r_t$ and terminal signal $d_t$
	 \STATE Store $(s_t, a_t, r_t, s_{t+1}, d_t) \to \calD$
	 \IF{it's time to update}
	 \FOR{however many updates}
	 \STATE Sample minibatch $B = \{(s_i,a_i,r_i,s'_i,d_i)\}$ from $\calD$
	 \STATE For each transition in $B$, compute TD errors:
	 \begin{equation*}
	 \Delta_i = r_i + \gamma (1-d_i) Q_{\theta}(s'_i, \mu_{\phi}(s'_i)) - Q_{\theta}(s_i,a_i)
	 \end{equation*}
	 \STATE Compute minibatch $K_{\theta}$ matrix and find least-squares solution $Z$ to $K_{\theta} Z = \Delta$
	 \STATE Compute proposed update for $Q$ with:
	 \begin{equation*}
	 \theta' = \theta + \alpha_q \sum_{(s,a) \in B} Z(s,a) \nabla_{\theta} Q_{\theta}(s,a)
	 \end{equation*}
	 \STATE Exponentially decay $\theta'$ towards $\theta$ until
	 \begin{equation*}
	 \cos\left( Q_{\theta'} - Q_{\theta}, \calT^*Q_{\theta} - Q_{\theta}\right) \geq \eta,
	 \end{equation*}
	 then set $\theta \leftarrow \theta'$.
	 \STATE Update $\mu$ with:
	 \begin{equation*}
	 \phi \leftarrow \phi + \alpha_{\mu} \frac{1}{|B|}\sum_{s \in B} \nabla_{\phi} Q_{\theta}(s, \mu_{\phi}(s))
	 \end{equation*}
	 \ENDFOR
	 \ENDIF
	\ENDFOR
\end{algorithmic}
\end{algorithm}

\subsection{Connection to Natural Gradients}

As it turns out, PreQN is equivalent to natural gradient Q-learning (NGQL) when the same samples are used to form both the gradient and the Fisher information matrix. To recap, the update for NGQL is
\begin{equation}
\theta' = \theta + \alpha F^{-1}_{\theta} g, \label{ngql}
\end{equation}
where $g$ is the gradient from Eq \ref{deepq} and
\begin{equation}
F_{\theta} = \underE{s,a \sim \rho}{\nabla_{\theta} Q_{\theta}(s,a) \nabla_{\theta}Q_{\theta}(s,a)^T} \label{fisher}
\end{equation}
is the Fisher information matrix for a gaussian distribution, $\calN(Q_{\theta},I)$. When using sample estimates of the expectations, we can write the NGQL update in terms of the matrix $\Phi_{\theta}$ (the $d \times |S||A|$ matrix with columns $\nabla_{\theta}Q_{\theta}(s,a)$) and the vector of TD-errors $\Delta = \calT^*Q_{\theta} - Q_{\theta}$ as:
\begin{equation}
\theta' = \theta + \alpha (\Phi_{\theta} \Phi_{\theta}^T)^\dagger \Phi_{\theta} \Delta, \label{ngql1}
\end{equation}
where $(\Phi_{\theta} \Phi_{\theta}^T)^\dagger$ is the pseudoinverse of $\Phi_{\theta} \Phi_{\theta}^T$. Similarly, the PreQN update as described by Eqs \ref{preqn1} and \ref{preqnupdate} can be written as
\begin{equation}
\theta' = \theta + \alpha \Phi_{\theta} (\Phi_{\theta}^T \Phi_{\theta})^\dagger \Delta. \label{preqn2}
\end{equation}
By the following lemma, the two updates in Eqs \ref{ngql1} and \ref{preqn2} are equivalent:
\begin{restatable}{lemma}{lemufour} $(\Phi_{\theta} \Phi_{\theta}^T)^\dagger \Phi_{\theta} = \Phi_{\theta} (\Phi_{\theta}^T \Phi_{\theta})^\dagger$.
\end{restatable}
The connection between NGQL and approximately non-expansive $Q$-update operators may explain the finding by \cite{Knight2018} that NGQL did not require target networks to remain stable. A related observation was made by \cite{Schulman2017}, who showed that a natural policy gradient could be viewed as approximately applying a version of Eq \ref{update1} for entropy-regularized Q-learning. They also demonstrated a version of DQL that could learn stably without target networks.

\section{Experiments}

In our experiments, we investigated the following questions:
\begin{enumerate}
\item What insights can we obtain about the neural tangent kernel $K_{\theta}$ in the context of RL? Can we exploit empirical analyses of $K_{\theta}$ to make better decisions about neural network architecture?
\item How does PreQN behave? To evaluate performance, we compare to TD3 \cite{Fujimoto2018} and SAC \cite{Haarnoja2018} on various OpenAI Gym \cite{Brockman2016} environments.
\item To what degree does a standard DQL update push $Q$-values towards their targets? How does this change with architecture? How should we interpret PreQN results in light of this?
\end{enumerate}

\subsection{Neural Tangent Kernel Analysis}

Based on Theorem 2, we are motivated to empirically evaluate two properties of the neural tangent kernel that appear relevant to stability in DQL: the magnitudes of diagonal elements, and the degree of off-diagonalness. To measure the latter, we consider the ratio of the average off-diagonal row entry to the on-diagonal entry, $R_i$:
\begin{align*}
R_i (K) &\doteq \frac{1}{N}\frac{\sum_{j\neq i} |K_{ij}|}{K_{ii}},
\end{align*}
where $N$ is the size of the square matrix $K$. We refer to this quantity as the `row ratio.'

We evaluate the standard neural networks used in DQL for continuous control: namely, feedforward multi-layer perceptrons with between 1 and 4 hidden layers, and between 32 and 256 hidden units, with either tanh, relu, or sin activations. (See, eg, \cite{Lillicrap2016, Fujimoto2018, Haarnoja2018, Rajeswaran2017} for examples where networks in this range have previously been used.) Because divergence typically occurs near the beginning of training \cite{VanHasselt2018} and the NTK is known to converge to a constant in the infinite-width limit \cite{Jacot2018}, we focus only on the properties of $K_{\theta}$ at initialization in these experiments.

For each of three Gym environments (HalfCheetah-v2, Walker2d-v2, and Ant-v2), we sampled a dataset $\calD$ of 1000 state-action pairs using a ``rails-random" policy: $a = \sign(u), u \sim \text{Unif}(\calA)$. We then randomly initialized neural networks of various sizes and activation functions, computed $K_{\theta}$ for each using the state-action pairs in $\calD$, and evaluated their on-diagonal elements and average row ratios. We show partial results in Figure \ref{ntkfig}, complete results in Appendix \ref{completentk}, and summarize findings here:
\begin{itemize}
\item Diagonal elements tend to increase with width and decrease with depth, across activation functions.
\item Row ratios tend to increase with depth across activation functions, and do not clearly correlate with width.
\item Relu nets commonly have the largest on-diagonal elements and row ratios (so they should learn quickly and generalize aggressively). 
\item Sin networks appear to be in a ``sweet spot" of high on-diagonal elements and low off-diagonal elements. This analysis suggests sin activations may be more useful for DQL than has been previously realized.
\end{itemize}

Based on these results, as we will detail in subsequent sections, we experimented with using sin activations for TD3, SAC, and PreQN. 

\begin{figure}
\centering
\includegraphics[width=0.4\textwidth]{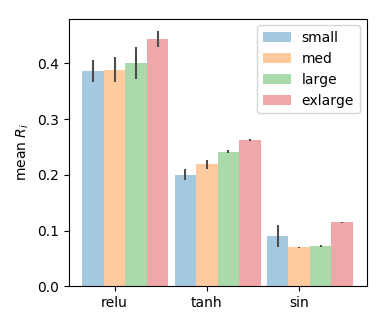}
\caption{Average row ratio for networks with 2 hidden layers of size 32 (small), 64 (med), 128 (large), and 256 (exlarge), using data from Walker2d-v2. Error bars are standard deviations from 3 random network initializations (with fixed data).}
\label{ntkfig}
\end{figure}

\subsection{Benchmarking PreQN}

\begin{figure*}[t]
\centering
\includegraphics[width=0.3\textwidth]{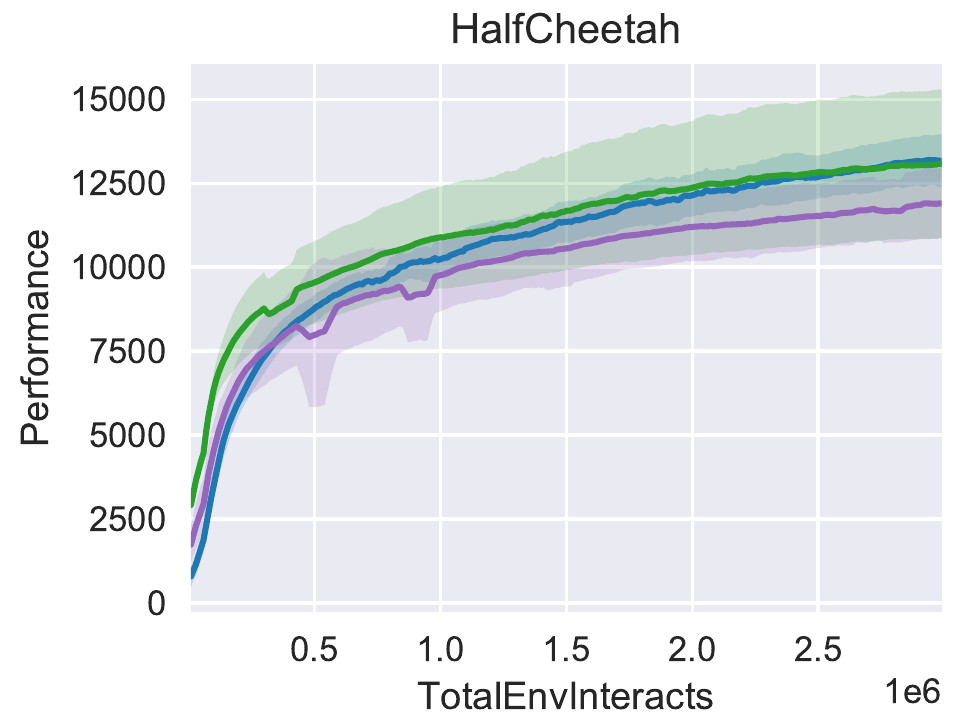}
\includegraphics[width=0.3\textwidth]{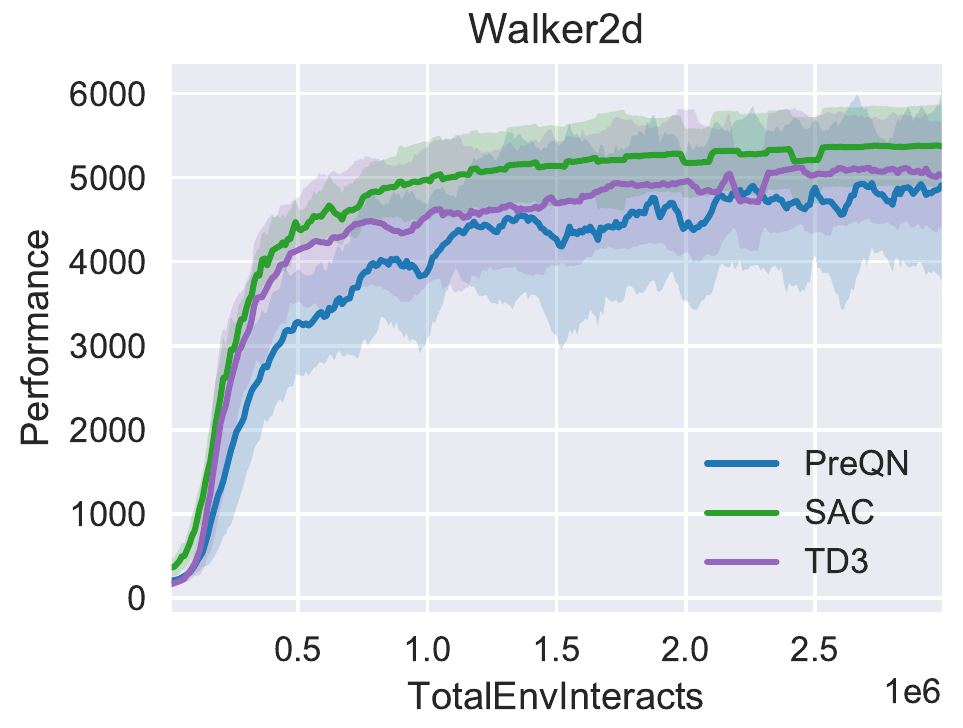}
\includegraphics[width=0.3\textwidth]{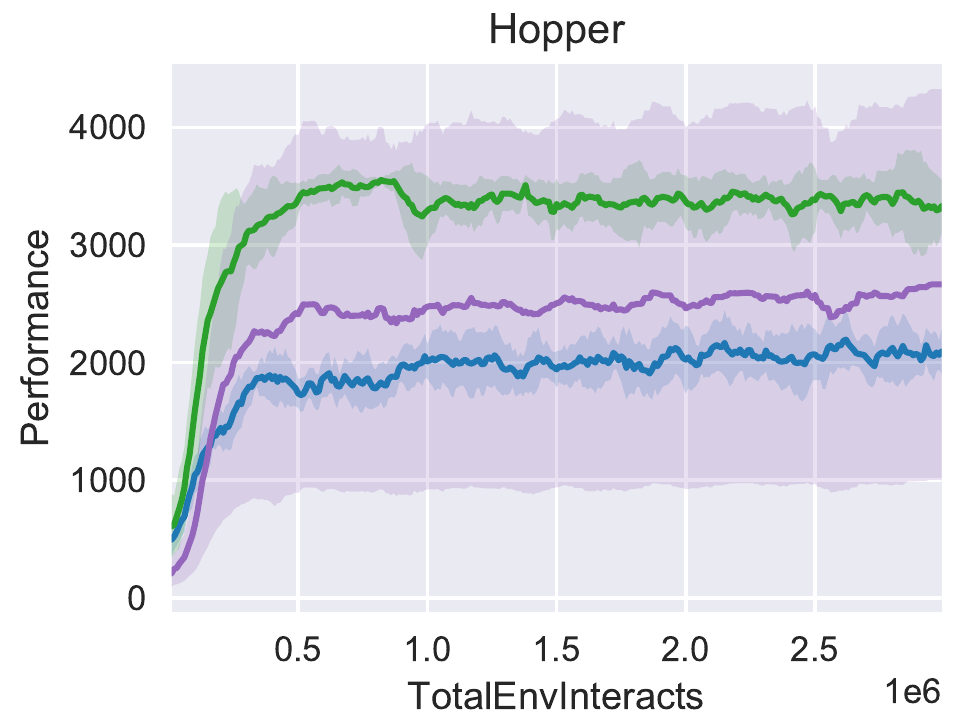}

\includegraphics[width=0.3\textwidth]{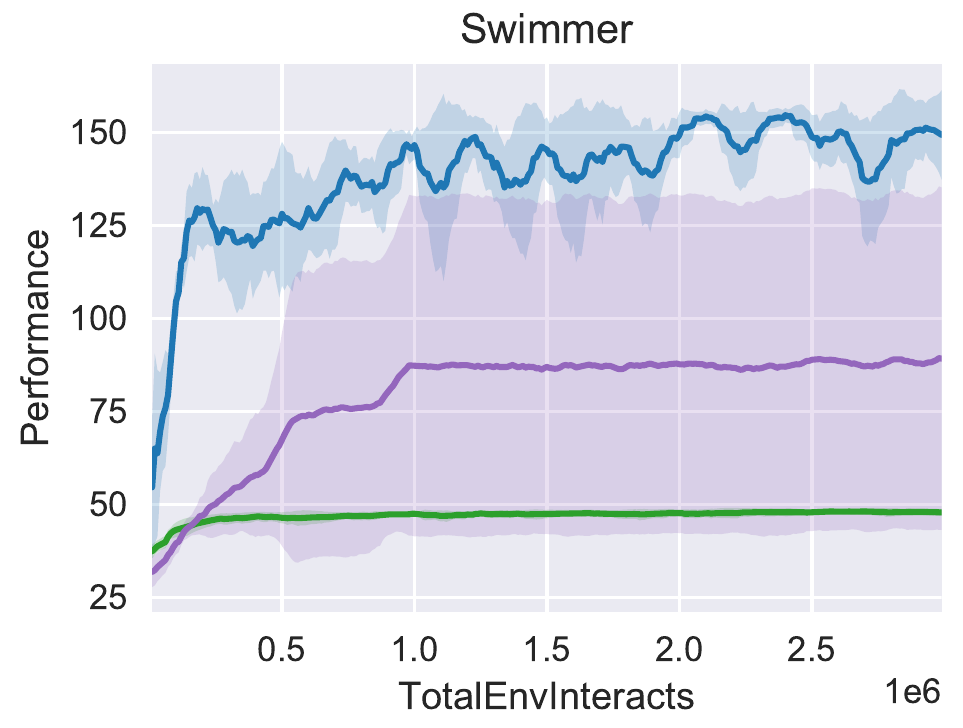}
\includegraphics[width=0.3\textwidth]{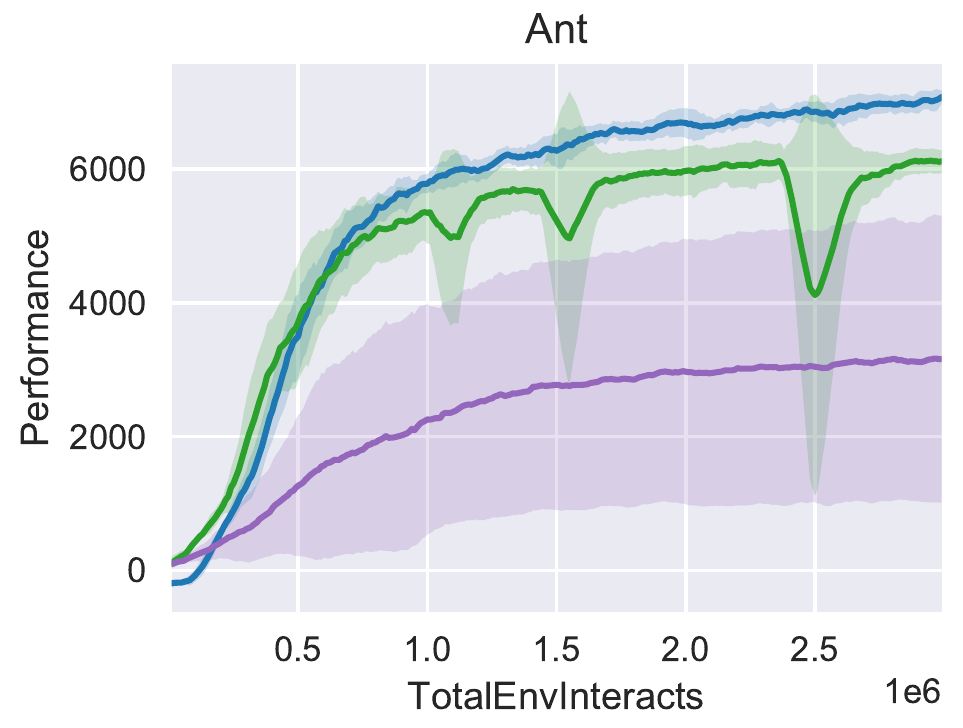}
\caption{Benchmarking PreQN against TD3 and SAC on standard OpenAI Gym MuJoCo environments. Curves are averaged over 7 random seeds. PreQN is stable and performant, despite not using target networks. The PreQN experiments used sin activations; the TD3 and SAC experiments used relu activations.}
\label{bench}
\end{figure*}

We benchmarked PreQN on five environments from the OpenAI Gym, comparing to TD3 and fixed-temperature SAC, and we present results in Figure \ref{bench}. For each algorithm, we experimented with using relu and sin activations, and we found that PreQN performed best with sin activations, while TD3 and SAC performed best with relu activations. (As a result, we report results in our benchmark for PreQN-sin, TD3-relu, and SAC-relu.) However, we did not do any hyperparameter tuning for TD3 and SAC to specifically accomodate the sin activations, and instead relied on hyperparameters based on the literature which were well-tuned for relus. Hyperparameters for all experiments are given in Appendix \ref{benchmethods}.

In general, PreQN is stable and performant, comparing favorably with the baseline algorithms. In some cases it outperforms (eg Swimmer and Ant) and in some cases it underperforms (eg Hopper). We find this outcome interesting and exciting because PreQN represents a different development path for DQL algorithms than is currently standard: it lacks target networks, only uses a single Q-function instead of multiple, makes no modifications to the bootstrap, and uses vanilla gradient steps for Q-learning instead of adaptive or momentum-based optimizers like Adam \cite{Kingma2015}. However (as we will shortly discuss), we found that it did not fully avert divergence when combined with relu networks.

\begin{figure}
\centering
\includegraphics[width=0.23\textwidth]{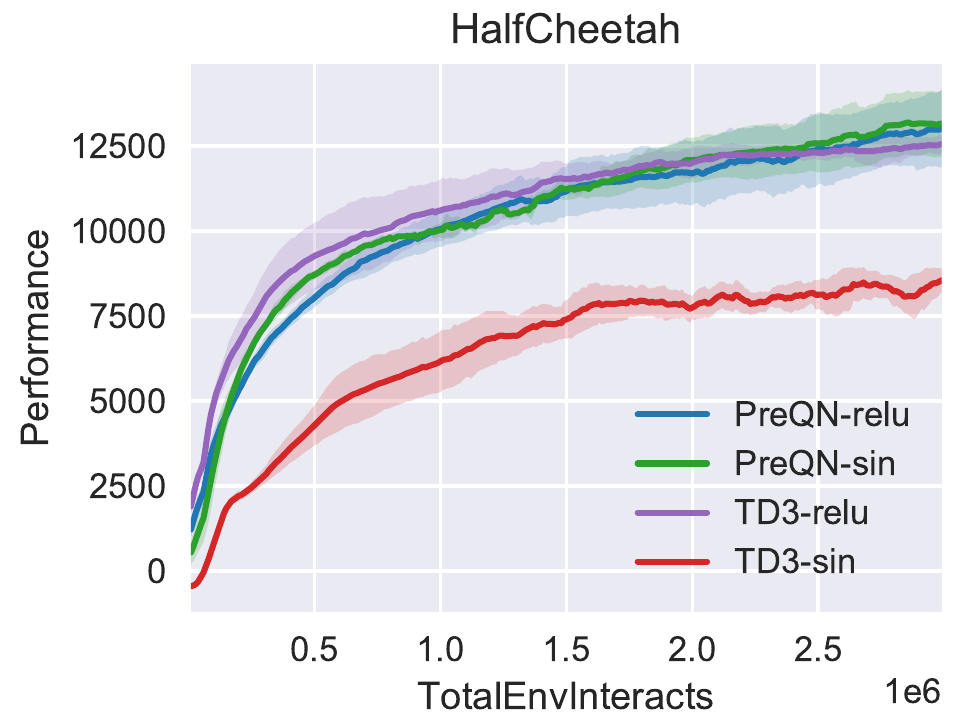}
\includegraphics[width=0.23\textwidth]{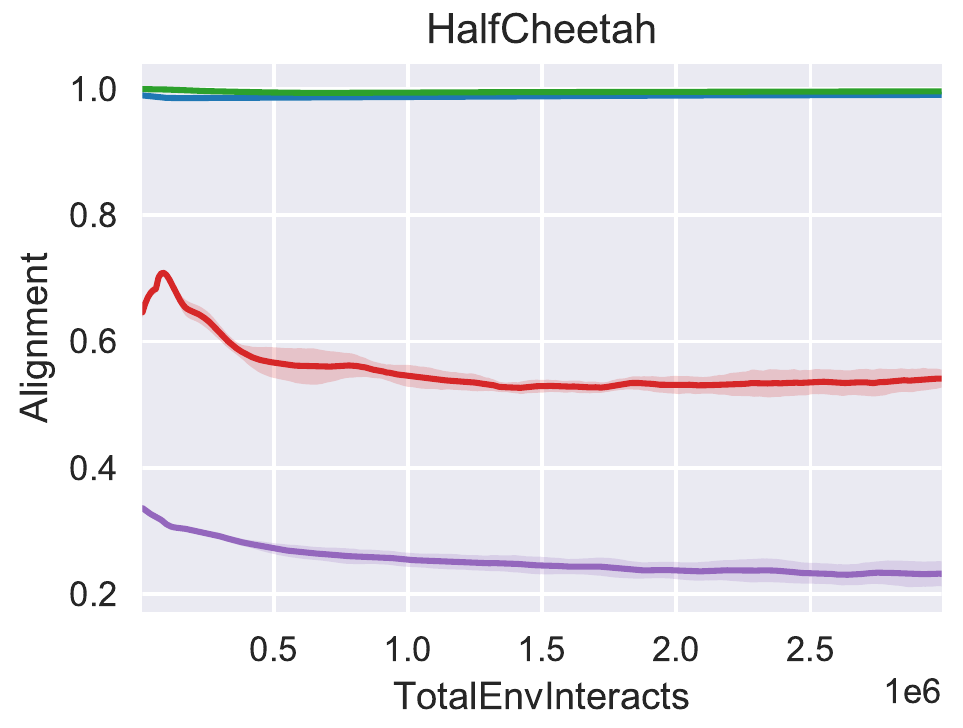}

\caption{Examining the cosine alignment of actual $Q$-value change with intended $Q$-value change ($\cos(Q'-Q, y -Q)$) for PreQN and TD3 with relu and sin activations. Curves are averaged over 3 random seeds.}
\label{align}
\end{figure}

To measure how well DQL updates push $Q$-values towards their targets, we evaluated an alignment measure given by $\cos(Q'-Q, y - Q)$, where $y$ is the target for the algorithm ($\calT^* Q_{\theta}$ in PreQN, and $\calT^* Q_{\psi}$ in TD3, where $\psi$ are the parameters of the slowly-changing target network). We show partial results in Figure \ref{align} and complete results in Appendix \ref{alignfull}. We compared PreQN to TD3, because these two algorithms are equivalent except for the $Q$-learning component. While PreQN produces high alignment regardless of architecture by design, TD3 with the sin function (TD3-sin) produces updates that are better-aligned with their targets than TD3 with the relu function (TD3-relu). This accords well with our empirical analysis of the NTK: for sin networks, the NTK is closer to diagonal, so $Q' - Q \approx \alpha K (y-Q)$ is closer to $\alpha (y-Q)$. Perhaps surprisingly, performance for TD3-sin was generally weaker than performance for TD3-relu, but we did not retune any of the hyperparameters from TD3-relu for TD3-sin; we speculate that better performance with TD3-sin may be achievable with a target network that updates more quickly. Performance for PreQN-relu was generally weaker than PreQN-sin, primarily due to occasional divergence; this result suggests that cancelling within-batch generalization is not a universal solution to divergence issues, and favorable architecture choices may be useful. However, in experiments not included in this report, we found that divergence issues with PreQN-relu were straightforwardly resolved by decreasing the learning rate (at a nontrivial cost to performance).

We are intrigued by the fact that empirical analysis of the NTK successfully predicts how the cosine alignment of a DQL update changes with architecture in the TD3 experiments. It has been observed that architecture changes can have a significant effect on performance in deep RL (for example, see \mbox{\citet{Henderson2018}}), but to the best of our knowledge, no one has previously proposed any method for predicting how the behavior of a given algorithm might change with architecture. Based on our results, we are cautiously optimistic that the NTK is the correct object of study for such predictions, and we recommend a more rigorous empirical analysis relating NTK measurements, architectures, and hyperparameter choices in DQL to performance.

\subsection{Remarks on Computational Cost}

Our implementation of PreQN was significantly slower than our implementation of SAC (by more than 50\%), due to the requirement of calculating backward passes separately for each state-action pair in the batch, and solving the system of equations $K_{\theta} Z = \Delta$. However, we consider it plausible that many adjustments could be made to reduce computational cost from our basic code. (For instance: we did not reuse the gradients from computing $K_{\theta}$ for forming the update, $\sum_{s,a} Z(s,a) \nabla_{\theta} Q_{\theta}(s,a)$, and this redundancy can be eliminated.)

\section{Other Related Work}

Previously, \citet{Melo2008} proved sufficient conditions for the convergence of $Q$-learning with linear function approximators. Their conditions were fairly restrictive, essentially requiring that the algorithm behave as if it were on-policy---removing one of the components of the triad. We see an interesting parallel to our results for the linear approximation case (Theorem 2), which also effectively remove a component of the triad by requiring the algorithm to behave as if it were tabular.

Concurrently to our work, \cite{Bhatt2019} developed CrossNorm, a variant of DDPG that uses a careful application of BatchNorm \cite{Ioffe2015} to achieve stable learning without target networks. Also concurrently, \citet{Fu2019} performed a rigorous empirical study of Fitted Q-Iteration (FQI) \cite{Riedmiller2005} to gain insight into divergence issues in DQL, and ultimately proposed an algorithm based on data distribution modifications to improve performance.

\section{Conclusions}

In this work, we examined how $Q$-values change under a DQL update in order to understand how divergence might arise. We used our insights to develop a practical algorithm, called PreQN, which attacks one of the root causes of divergence: the generalization properties of the $Q$-function approximator, as quantified by the neural tangent kernel \cite{Jacot2018}. Our experiments show that PreQN, with appropriate design choices, is stable and performant on various high-dimensional continuous control tasks.

Intriguingly, theoretical and empirical work shows that the NTK converges to a constant (independent of network parameters) in the limit of wide networks \cite{Jacot2018}; this result makes it possible to study the evolution of neural network functions through their linearization around the starting point \cite{Lee2019}. In this regime, through the correspondence in Lemma 3, DQL should behave quite closely to linear TD-learning. We consider the detailed analysis of this connection to be an interesting avenue for potential future work.

\bibliography{preqn}
\bibliographystyle{icml2019}

\onecolumn
\appendix
\section{Proofs} \label{appendix}

\lemuone*

\begin{proof} To establish that $\calU_1$ is a contraction, we compute:
\begin{align*}
\|\calU_1 Q_1 - \calU_1 Q_2\|_{\infty} &= \|(1-\alpha)(Q_1 - Q_2) + \alpha (\calT^* Q_1 - \calT^* Q_2)\|_{\infty} \\
&\leq (1-\alpha)\|Q_1 - Q_2\|_{\infty} + \alpha\|\calT^* Q_1 - \calT^* Q_2\|_{\infty} \\
&\leq (1-\alpha)\|Q_1 - Q_2\|_{\infty} + \alpha \gamma \| Q_1 - Q_2\|_{\infty} \\
&= (1 - (1-\gamma)\alpha) \|Q_1 - Q_2\|_{\infty}.
\end{align*}
Thus $\calU_1$ contracts with modulus $1 - (1-\gamma)\alpha < 1$. That $Q^*$ is the fixed-point follows immediately from $\calT^* Q^* = Q^*$. 

\end{proof}

\lemutwo*

\begin{proof}
First, we observe that for any $s,a$, we have:
\begin{align*}
[\calU_2 Q_1 - \calU_2 Q_2](s,a) &= \big(1 - \alpha \rho(s,a)\big)\bigg( Q_1 (s,a) - Q_2 (s,a)\bigg) + \alpha \rho(s,a) \bigg( \left[\calT^*Q_1 - \calT^* Q_2\right](s,a) \bigg) \\
&\leq \big(1 - \alpha \rho(s,a)\big)\|Q_1 - Q_2\|_{\infty} + \alpha \gamma \rho(s,a) \|Q_1 - Q_2\|_{\infty} \\
&= \big(1 - (1-\gamma)\alpha \rho(s,a)\big) \|Q_1 - Q_2\|_{\infty}.
\end{align*}
Then, by taking the max over $(s,a)$ on both sides (first the right, and then the left), we obtain:
\begin{align*}
\|\calU_2 Q_1 - \calU_2 Q_2\|_{\infty} &\leq \max_{s,a} \big(1 - (1-\gamma)\alpha \rho(s,a)\big) \|Q_1 - Q_2\|_{\infty} \\
&= \left(1 - (1-\gamma)\alpha \rho_{min}\right)\|Q_1 - Q_2\|_{\infty},
\end{align*}
where $\rho_{min} = \min_{s,a} \rho(s,a)$. If $\rho_{min} > 0$ (which is equivalent to the condition $\forall s,a, \rho(s,a) > 0$), then the modulus $\left(1 - (1-\gamma)\alpha \rho_{min}\right) < 1$ and $\calU_2$ is a contraction. That its fixed-point is $Q^*$ follows from $\calT^* Q^* = Q^*$.

However, if $\rho_{min} = 0$, we merely have an upper bound on $\|\calU_2 Q_1 - \calU_2 Q_2\|_{\infty}$ and that alone is not sufficient to demonstate that $\calU_2$ is a not a contraction. This is easy to demonstrate without resorting to inequalities, though: for any $s,a$ with $\rho(s,a) = 0$, it is straightforward to see that $\calU_2 Q (s,a) = Q(s,a)$---that is, $\calU_2$ leaves that $Q$-value unchanged. Thus there are choices of $Q_1, Q_2$ such that $\|\calU_2 Q_1 - \calU_2 Q_2\|_{\infty} = \|Q_1 - Q_2\|_{\infty}$ and $\calU_2$ is a non-expansion.
\end{proof}

\lemuthree*

\begin{proof}
For the linear function approximation case, $\nabla_{\theta} Q_{\theta}(s,a) = \phi(s,a)$, and the NTK (Eq. \ref{kernelcomponent}) therefore has components
\begin{equation*}
K_{\theta}(\bar{s},\bar{a},s,a) = \phi(\bar{s},\bar{a})^T \phi(s,a),
\end{equation*}
which are independent of $\theta$. Thus $K_{\theta} = K$. 

Furthermore, $\nabla^2_{\theta} Q_{\theta} (s,a) = 0$ for all $s,a$ and all $\theta$, so all higher-order terms are zero. 
\end{proof}

\thmtwo*

\begin{proof}
Using index notation instead of tracking state-action pairs,
\begin{align*}
[\calU_3 Q_1 - \calU_3 Q_2]_i &= [Q_1 - Q_2]_i + \alpha \sum_j K_{ij} \rho_j \left[ (\calT^* Q_1 - Q_1) - (\calT^* Q_2 - Q_2)\right]_j \\
&= \sum_j \left(\delta_{ij} - \alpha K_{ij}\rho_j\right) [Q_1 - Q_2]_j + \alpha \sum_j K_{ij} \rho_j \left[ \calT^* Q_1 - \calT^* Q_2\right]_j \\
&\leq \sum_j \left(|\delta_{ij} - \alpha K_{ij}\rho_j| + \alpha \gamma | K_{ij}| \rho_j \right)\|Q_1 - Q_2 \|_{\infty}.
\end{align*}

Thus we can obtain a modulus as $\beta(K) = \max_i \sum_j \left(|\delta_{ij} - \alpha K_{ij}\rho_j| + \alpha \gamma | K_{ij}| \rho_j \right)$. We'll break it up into on-diagonal and off-diagonal parts, and assume that $\alpha K_{ii} \rho_i < 1$:
\begin{align*}
\beta(K) &= \max_i \sum_j \left(|\delta_{ij} - \alpha K_{ij}\rho_j| + \alpha \gamma | K_{ij}| \rho_j \right) \\
&= \max_i \left(\left( |1 - \alpha K_{ii}\rho_i| + \alpha \gamma K_{ii} \rho_i \right) + (1+\gamma) \alpha \sum_{j\neq i}  |K_{ij}| \rho_j \right) \\
&= \max_i \left(1 - (1-\gamma)\alpha K_{ii} \rho_i + (1+\gamma)\alpha \sum_{j\neq i}  |K_{ij}| \rho_j\right)
\end{align*}

A guarantee that $\beta(K) < 1$ can then be obtained by requiring that
\begin{equation*}
\forall i, \;\;\;\;\; (1+\gamma) \sum_{j\neq i}  |K_{ij}| \rho_j \leq (1-\gamma) K_{ii} \rho_i.
\end{equation*}

We note that this is a quite restrictive condition, since for $\gamma$ high, $(1+\gamma) / (1-\gamma)$ will be quite large, and the LHS has a sum over all off-diagonal terms in a row.

\end{proof}

\thmthree*

\begin{proof}
First, to obtain Eq. \ref{qconverge}:
\begin{align*}
\|\tilde{Q} - Q_i\| &= \|\calU_{i-1} \tilde{Q} - \calU_{i-1} Q_{i-1}\| && \text{Iterate sequence and fixed-point assumption}\\
&\leq \beta_{i-1} \|\tilde{Q} - Q_{i-1}\| && \text{Definition of Lipschitz continuity} \\
& \leq \left(\prod_{k=0}^{i-1} \beta_k\right) \|\tilde{Q} - Q_0\| && \text{Repeated application of above}
\end{align*}

That the sequence converges to $\tilde{Q}$ follows from 
\begin{equation*}
\lim_{N \to \infty} \prod_{k=j}^{N} \beta_k = 0,
\end{equation*}
when $\forall k \geq j, \; \beta_k \in [0,1)$.
\end{proof}

\lemufour*

\begin{proof} Let $\Phi_{\theta} \in \Real{m\times n}$ have rank $r$ and a singular value decomposition given by $\Phi_{\theta} = U \Sigma V^T$, with $\Sigma \in \Real{r \times r}$. Recall that $U^T U = V^T V = I_r$. Then $\Phi_{\theta} \Phi_{\theta}^T = U \Sigma^2 U^T$ and $\Phi_{\theta}^T \Phi_{\theta} = V \Sigma^2 V^T$, and:
\begin{align*}
(\Phi_{\theta} \Phi_{\theta}^T)^{\dagger} \Phi_{\theta} &= (U \Sigma^{-2} U^T) U \Sigma V^T \\
&= U \Sigma^{-1} V^T \\
&= U \Sigma V^T (V \Sigma^{-2} V^T) \\
&= \Phi_{\theta} (\Phi_{\theta}^T \Phi_{\theta})^{\dagger}
\end{align*}

\end{proof}

\section{Methods for PreQN Benchmark} \label{benchmethods}

We used the following hyperparameters for our PreQN benchmark experiments:

\begin{center}
\begin{tabular}{l|ll}
\hline
Hyperparameter & Value & \\ 
\hline
Discount factor $\gamma$ & $0.99$ \\
Batch size & $256$ & \\
Network size & $[256, 256]$ & \\
Actor learning rate & $10^{-3}$ & \\
Actor optimizer & Adam & \\
Critic learning rate & $10^{-3}$ for TD3 and SAC, & $0.1$ for PreQN \\
Update frequency & $50$ \\
Update after & $5000$ \\
Start steps & $5000$ \\
Alignment threshold $\eta$ (PreQN only) & $0.97$ \\
Action noise (TD3 and PreQN) & $\calN(0, 0.1)$\\
Target network polyak averaging (TD3 and SAC) & $0.995$ \\
Entropy regularization coefficient (SAC only) & $0.1$ \\
\hline
\end{tabular}
\end{center}

Here, `update frequency' refers to the number of environment steps that would elapse between occasions of updating the networks. During each update, there would be as many gradient descent steps (and target network polyak-averaging steps, if applicable) as environment steps had elapsed since the last update (so that for the overall training run, the ratio of gradient steps to env steps would be 1:1).

`Update after' refers to the number of steps that would elapse at the beginning of training before any gradient descent steps would take place (to allow time for filling the replay buffer).

For improved exploration, at the beginning of training agents would spend `start steps' number of steps acting under a uniform random policy.

TD3 hyperparameters for target policy smoothing and policy delay were taken from \cite{Fujimoto2018} without modification.

\raggedbottom
\pagebreak

\section{Extended Results for Neural Tangent Kernel Analysis} \label{completentk}

\subsection{Experiments with Network Width}
\begin{figure}[H]
\centering
\begin{subfigure}{}
\includegraphics[width=0.4\textwidth]{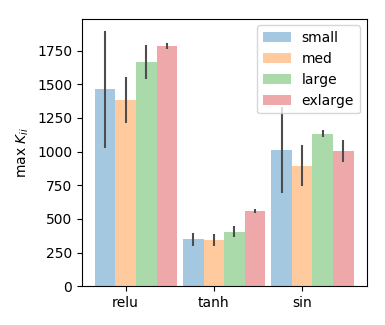}
\end{subfigure}
\begin{subfigure}{}
\includegraphics[width=0.4\textwidth]{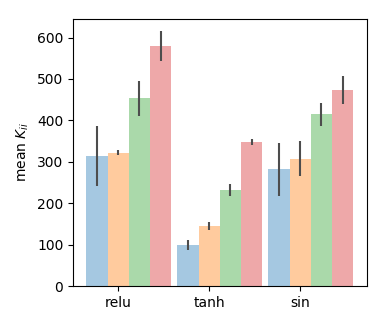}
\end{subfigure}

\begin{subfigure}{}
\includegraphics[width=0.4\textwidth]{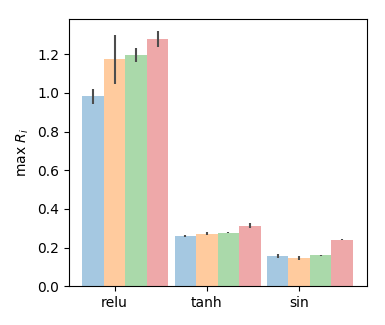}
\end{subfigure}
\begin{subfigure}{}
\includegraphics[width=0.4\textwidth]{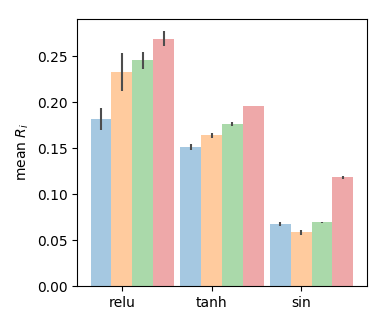}
\end{subfigure}

\caption{NTK analysis for randomly-initialized networks with various activation functions, where the NTKs were formed using 1000 steps taken by a rails-random policy in the \textbf{Ant-v2} gym environment (with the same data used across all trials). Networks are MLPs with widths of $32, 64, 128, 256$ hidden units (small, med, large, exlarge respectively) and $2$ hidden layers. Each bar is the average over 3 random trials (different network initializations).}
\end{figure}

\begin{figure}[H]
\centering
\begin{subfigure}{}
\includegraphics[width=0.4\textwidth]{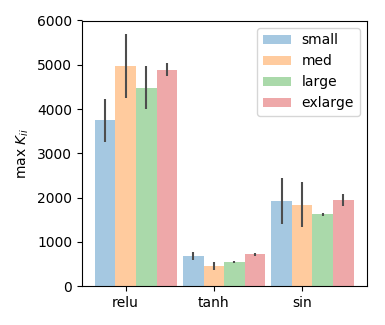}
\end{subfigure}
\begin{subfigure}{}
\includegraphics[width=0.4\textwidth]{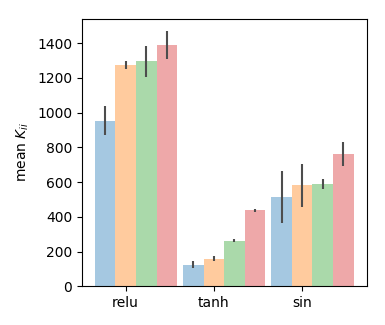}
\end{subfigure}

\begin{subfigure}{}
\includegraphics[width=0.4\textwidth]{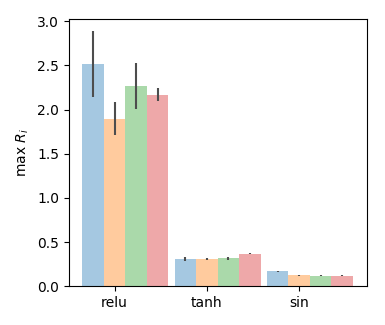}
\end{subfigure}
\begin{subfigure}{}
\includegraphics[width=0.4\textwidth]{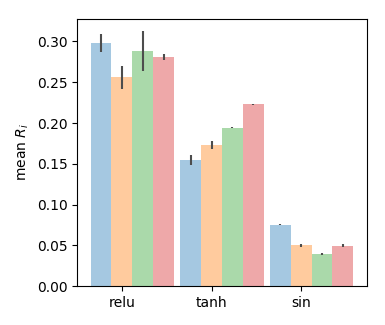}
\end{subfigure}

\caption{NTK analysis for randomly-initialized networks with various activation functions, where the NTKs were formed using 1000 steps taken by a rails-random policy in the \textbf{HalfCheetah-v2} gym environment (with the same data used across all trials). Networks are MLPs with widths of $32, 64, 128, 256$ hidden units (small, med, large, exlarge respectively) and $2$ hidden layers. Each bar is the average over 3 random trials (different network initializations).}
\end{figure}

\begin{figure}[H]
\centering
\begin{subfigure}{}
\includegraphics[width=0.4\textwidth]{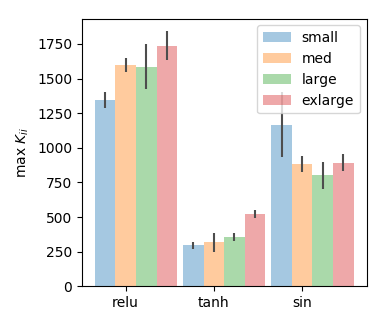}
\end{subfigure}
\begin{subfigure}{}
\includegraphics[width=0.4\textwidth]{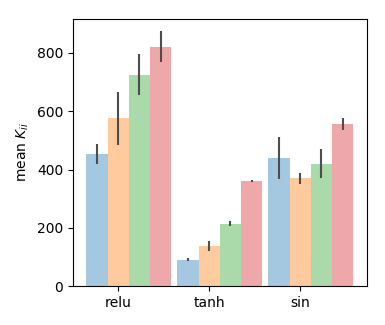}
\end{subfigure}

\begin{subfigure}{}
\includegraphics[width=0.4\textwidth]{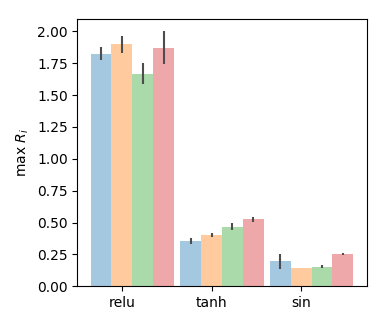}
\end{subfigure}
\begin{subfigure}{}
\includegraphics[width=0.4\textwidth]{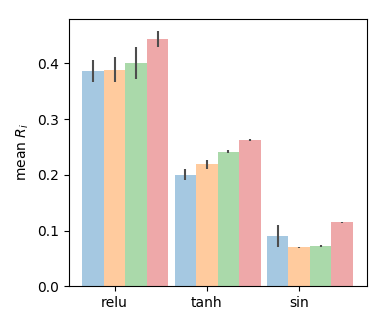}
\end{subfigure}

\caption{NTK analysis for randomly-initialized networks with various activation functions, where the NTKs were formed using 1000 steps taken by a rails-random policy in the \textbf{Walker2d-v2} gym environment (with the same data used across all trials). Networks are MLPs with widths of $32, 64, 128, 256$ hidden units (small, med, large, exlarge respectively) and $2$ hidden layers. Each bar is the average over 3 random trials (different network initializations).}
\end{figure}

\subsection{Experiments with Network Depth}
\begin{figure}[H]
\centering
\begin{subfigure}{}
\includegraphics[width=0.4\textwidth]{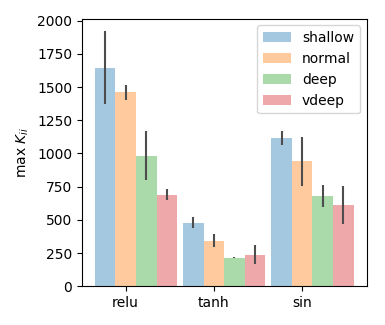}
\end{subfigure}
\begin{subfigure}{}
\includegraphics[width=0.4\textwidth]{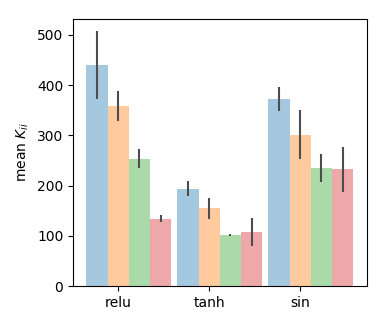}
\end{subfigure}

\begin{subfigure}{}
\includegraphics[width=0.4\textwidth]{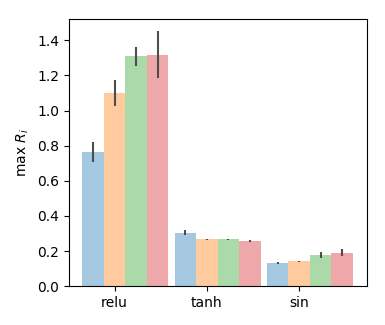}
\end{subfigure}
\begin{subfigure}{}
\includegraphics[width=0.4\textwidth]{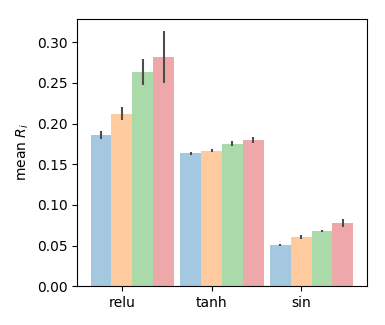}
\end{subfigure}

\caption{NTK analysis for randomly-initialized networks with various activation functions, where the NTKs were formed using 1000 steps taken by a rails-random policy in the \textbf{Ant-v2} gym environment (with the same data used across all trials). Networks are MLPs with depths of $1, 2, 3, 4$ hidden layers (shallow, normal, deep, vdeep respectively) and $64$ units per layer. Each bar is the average over 3 random trials (different network initializations).}
\end{figure}

\begin{figure}[H]
\centering
\begin{subfigure}{}
\includegraphics[width=0.4\textwidth]{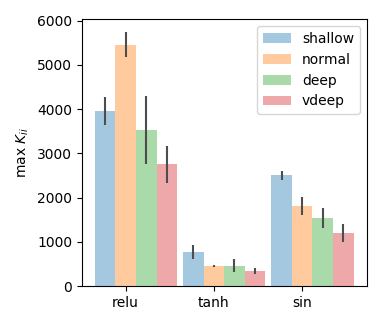}
\end{subfigure}
\begin{subfigure}{}
\includegraphics[width=0.4\textwidth]{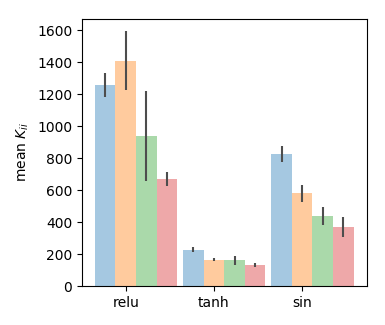}
\end{subfigure}

\begin{subfigure}{}
\includegraphics[width=0.4\textwidth]{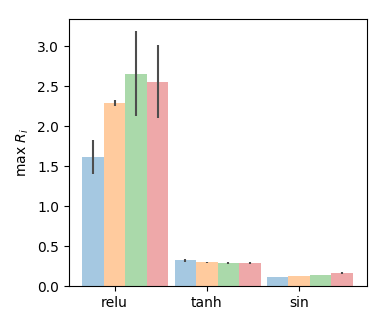}
\end{subfigure}
\begin{subfigure}{}
\includegraphics[width=0.4\textwidth]{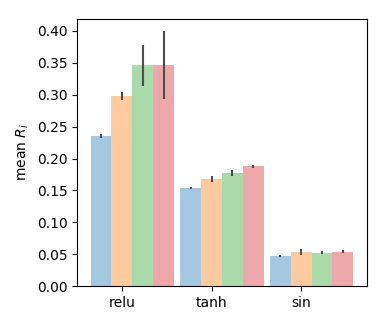}
\end{subfigure}

\caption{NTK analysis for randomly-initialized networks with various activation functions, where the NTKs were formed using 1000 steps taken by a rails-random policy in the \textbf{HalfCheetah-v2} gym environment (with the same data used across all trials). Networks are MLPs with depths of $1, 2, 3, 4$ hidden layers (shallow, normal, deep, vdeep respectively) and $64$ units per layer. Each bar is the average over 3 random trials (different network initializations).}
\end{figure}

\begin{figure}[H]
\centering
\begin{subfigure}{}
\includegraphics[width=0.4\textwidth]{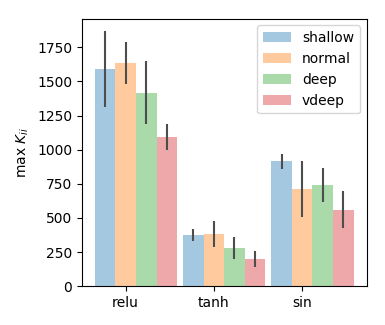}
\end{subfigure}
\begin{subfigure}{}
\includegraphics[width=0.4\textwidth]{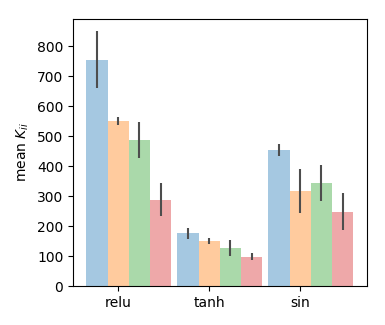}
\end{subfigure}

\begin{subfigure}{}
\includegraphics[width=0.4\textwidth]{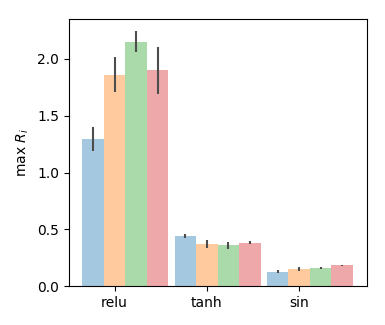}
\end{subfigure}
\begin{subfigure}{}
\includegraphics[width=0.4\textwidth]{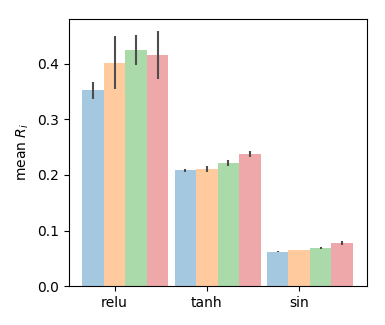}
\end{subfigure}

\caption{NTK analysis for randomly-initialized networks with various activation functions, where the NTKs were formed using 1000 steps taken by a rails-random policy in the \textbf{Walker2d-v2} gym environment (with the same data used across all trials). Networks are MLPs with depths of $1, 2, 3, 4$ hidden layers (shallow, normal, deep, vdeep respectively) and $64$ units per layer. Each bar is the average over 3 random trials (different network initializations).}
\end{figure}

\raggedbottom
\pagebreak

\section{Extended Results for Alignment Experiment with Architecture Ablation} \label{alignfull}

\begin{figure}[H]
\centering
\begin{subfigure}{}
\includegraphics[width=0.3\textwidth]{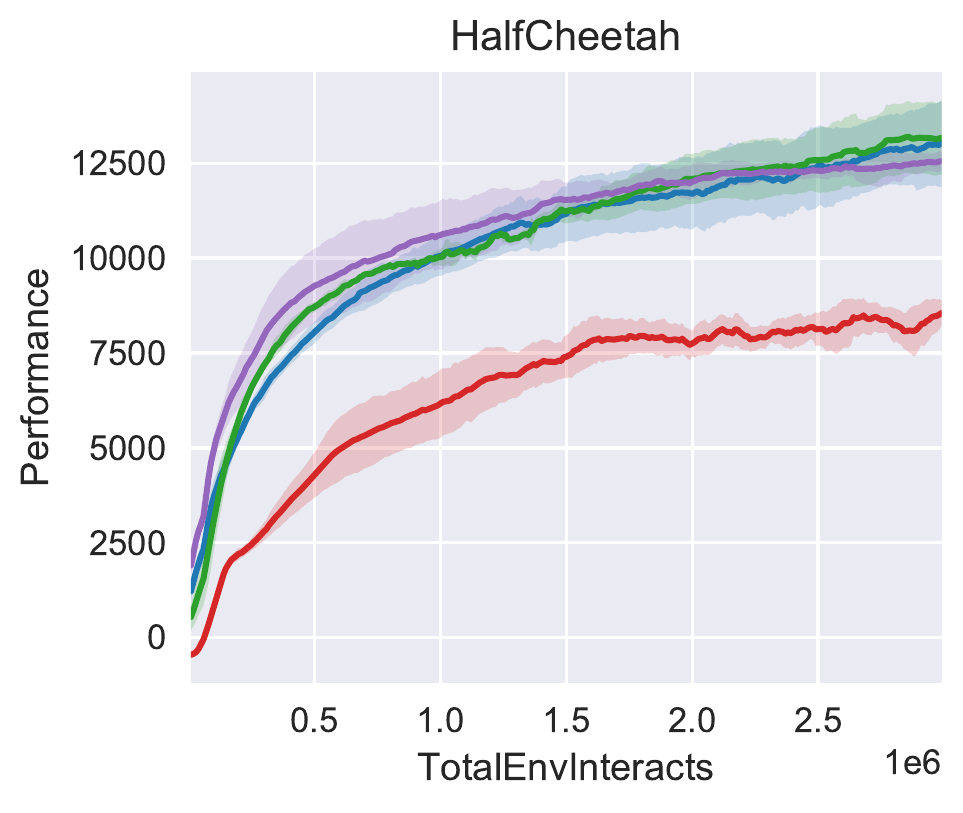}
\end{subfigure}
\begin{subfigure}{}
\includegraphics[width=0.3\textwidth]{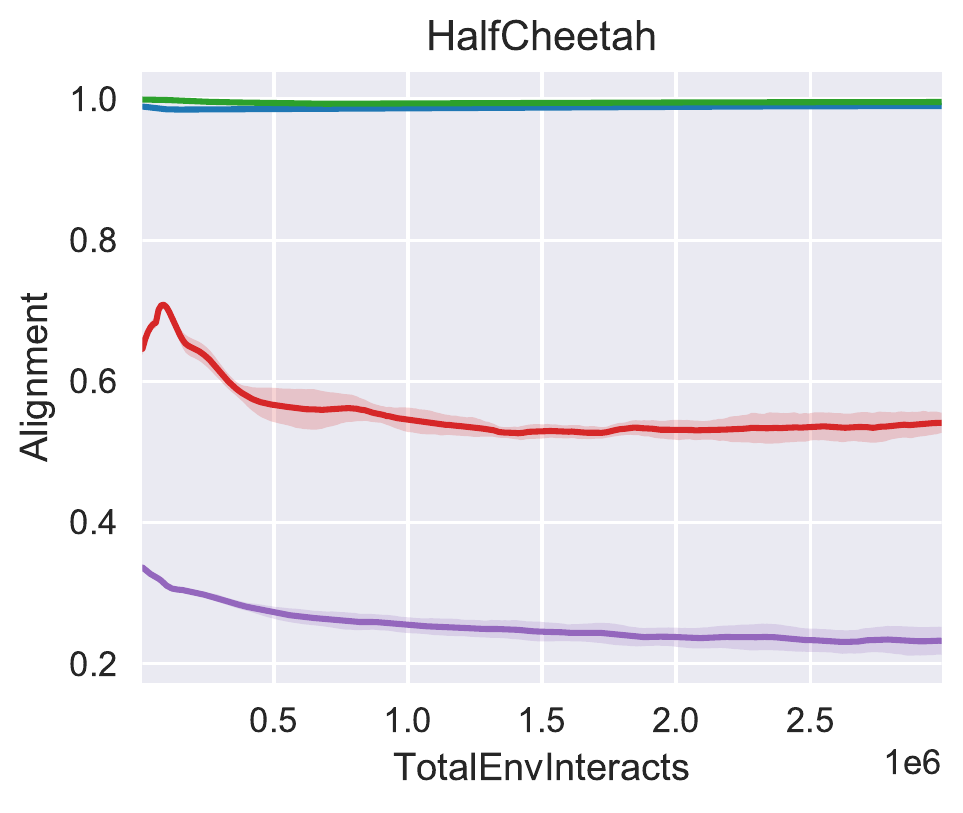}
\end{subfigure}
\begin{subfigure}{}
\includegraphics[width=0.3\textwidth]{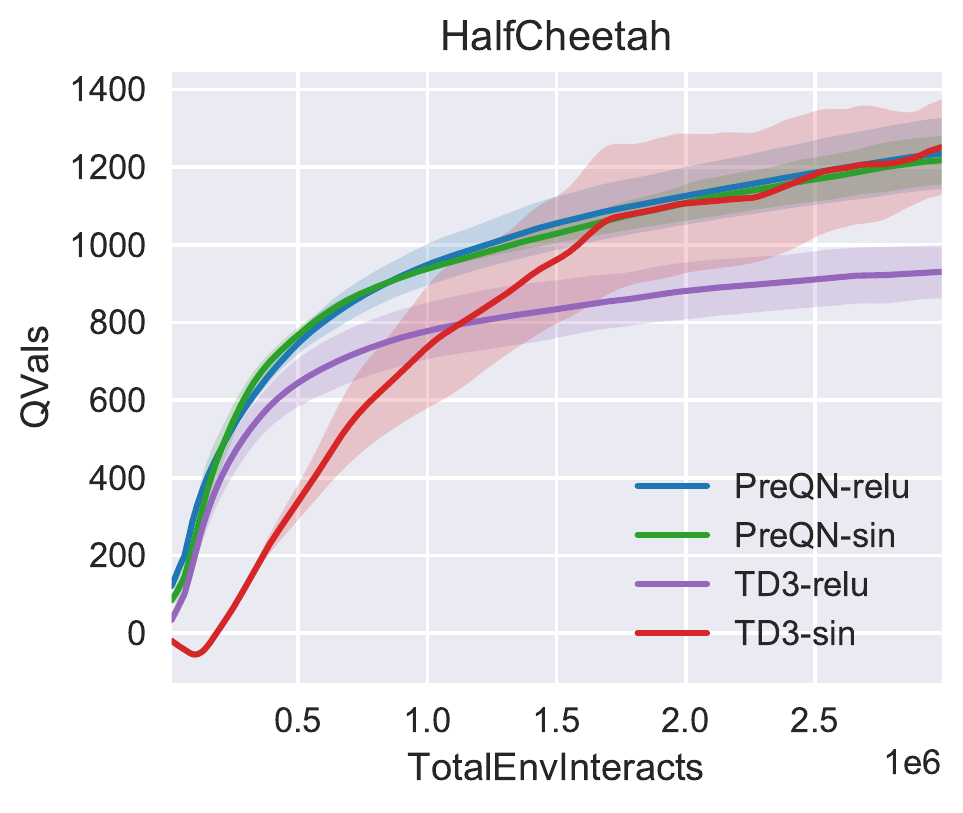}
\end{subfigure}
\caption{Comparison between PreQN and TD3 for relu and sin activation functions in the HalfCheetah-v2 gym environment. Results averaged over 3 random seeds.}
\end{figure}

\begin{figure}[H]
\centering
\begin{subfigure}{}
\includegraphics[width=0.3\textwidth]{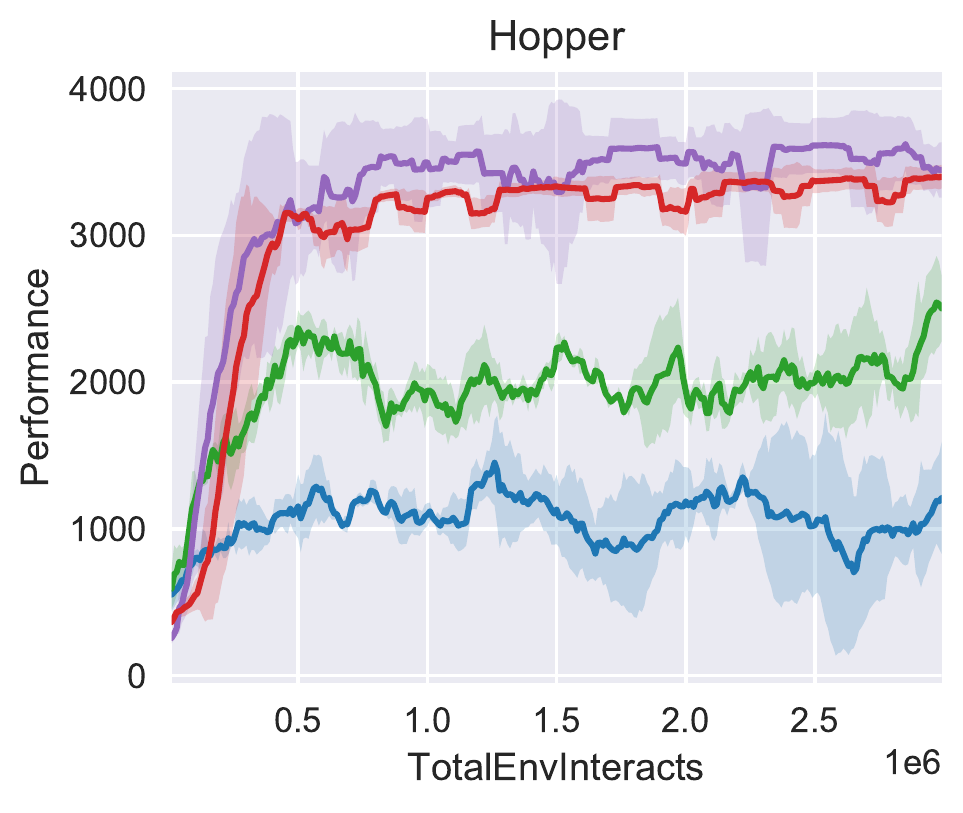}
\end{subfigure}
\begin{subfigure}{}
\includegraphics[width=0.3\textwidth]{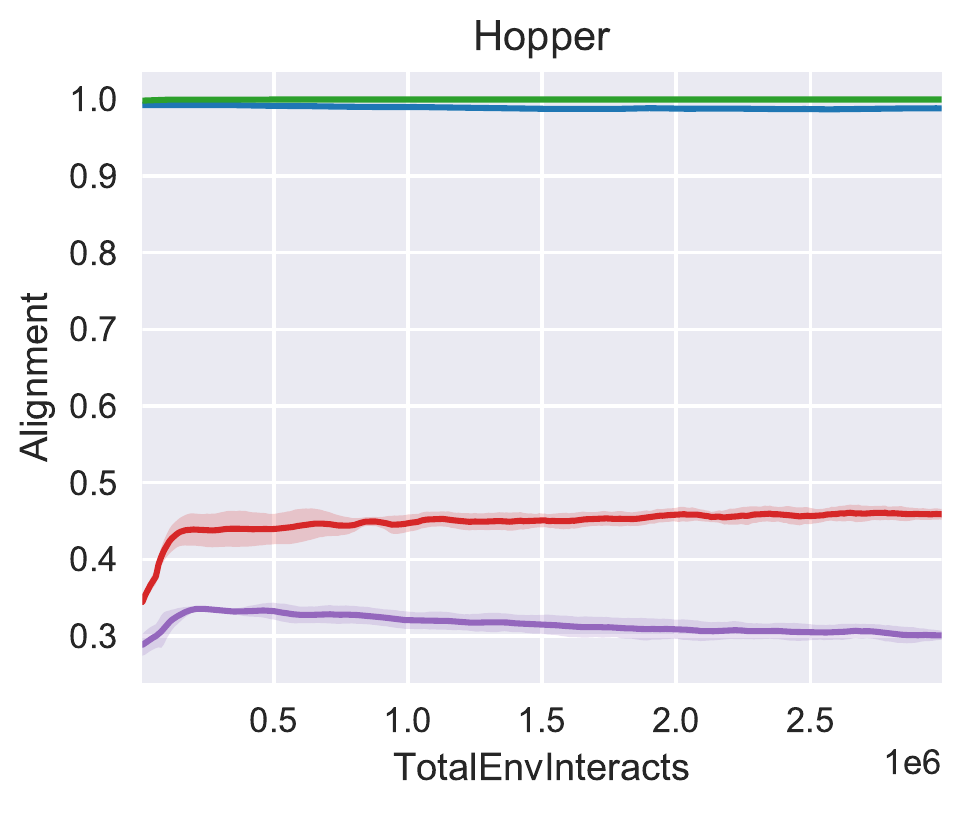}
\end{subfigure}
\begin{subfigure}{}
\includegraphics[width=0.3\textwidth]{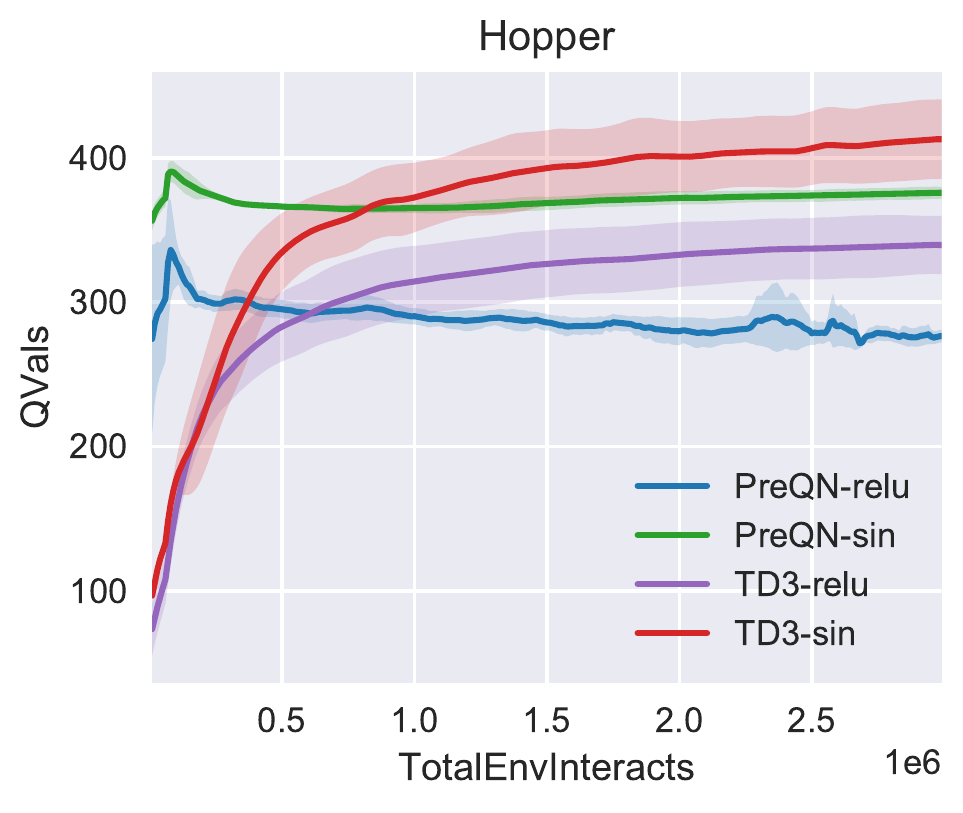}
\end{subfigure}
\caption{Comparison between PreQN and TD3 for relu and sin activation functions in the Hopper-v2 gym environment. Results averaged over 3 random seeds.}
\end{figure}

\begin{figure}[H]
\centering
\begin{subfigure}{}
\includegraphics[width=0.3\textwidth]{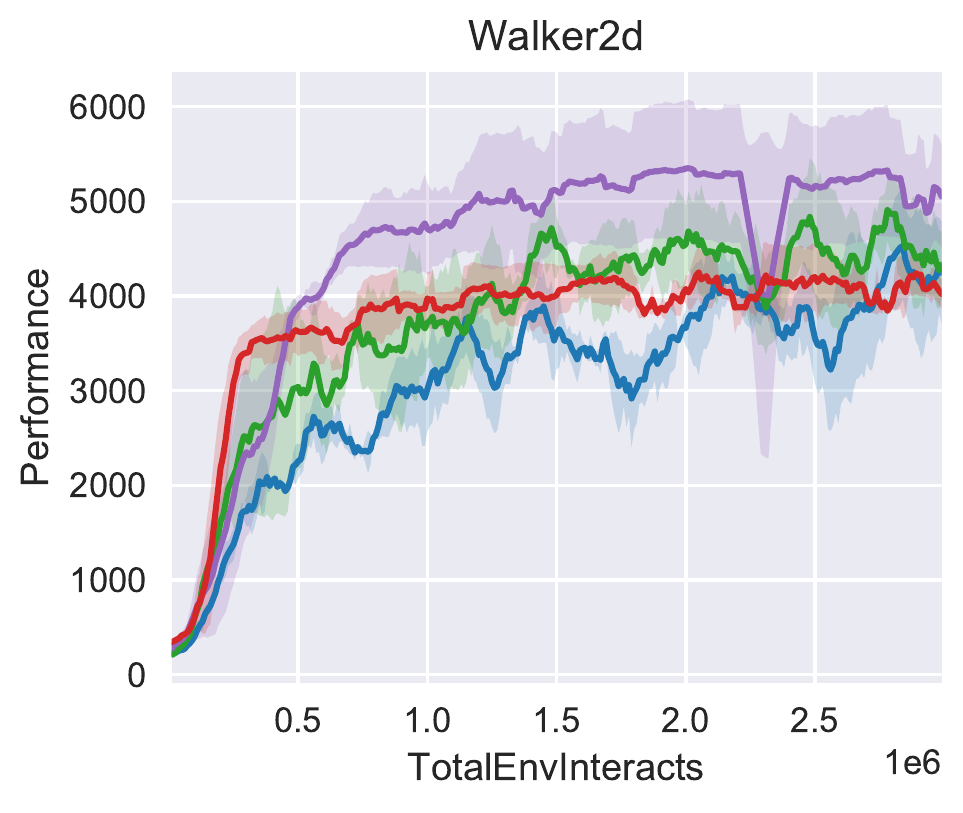}
\end{subfigure}
\begin{subfigure}{}
\includegraphics[width=0.3\textwidth]{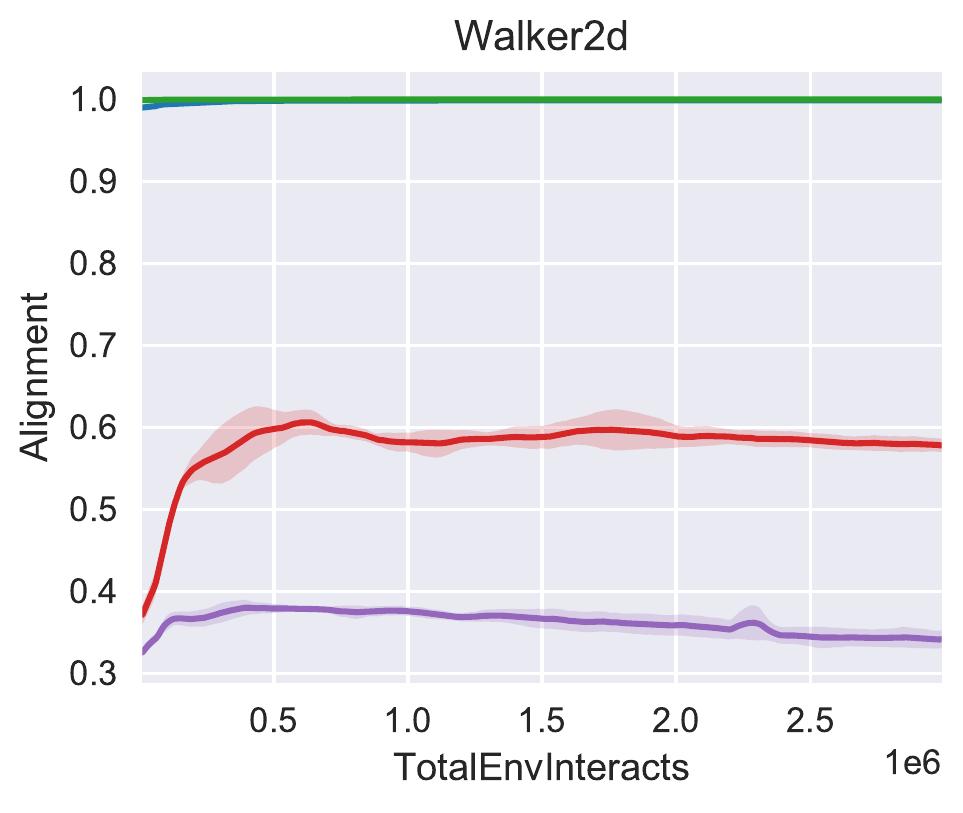}
\end{subfigure}
\begin{subfigure}{}
\includegraphics[width=0.3\textwidth]{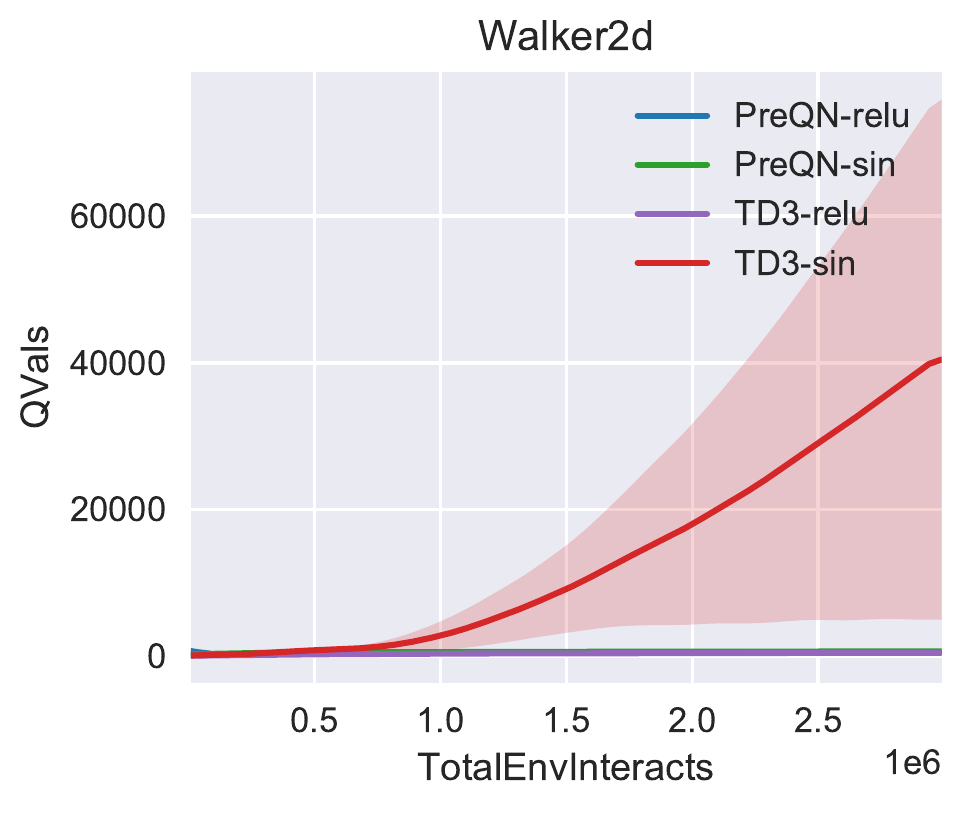}
\end{subfigure}
\caption{Comparison between PreQN and TD3 for relu and sin activation functions in the Walker2d-v2 gym environment. Results averaged over 3 random seeds.}
\end{figure}

\begin{figure}[H]
\centering
\begin{subfigure}{}
\includegraphics[width=0.3\textwidth]{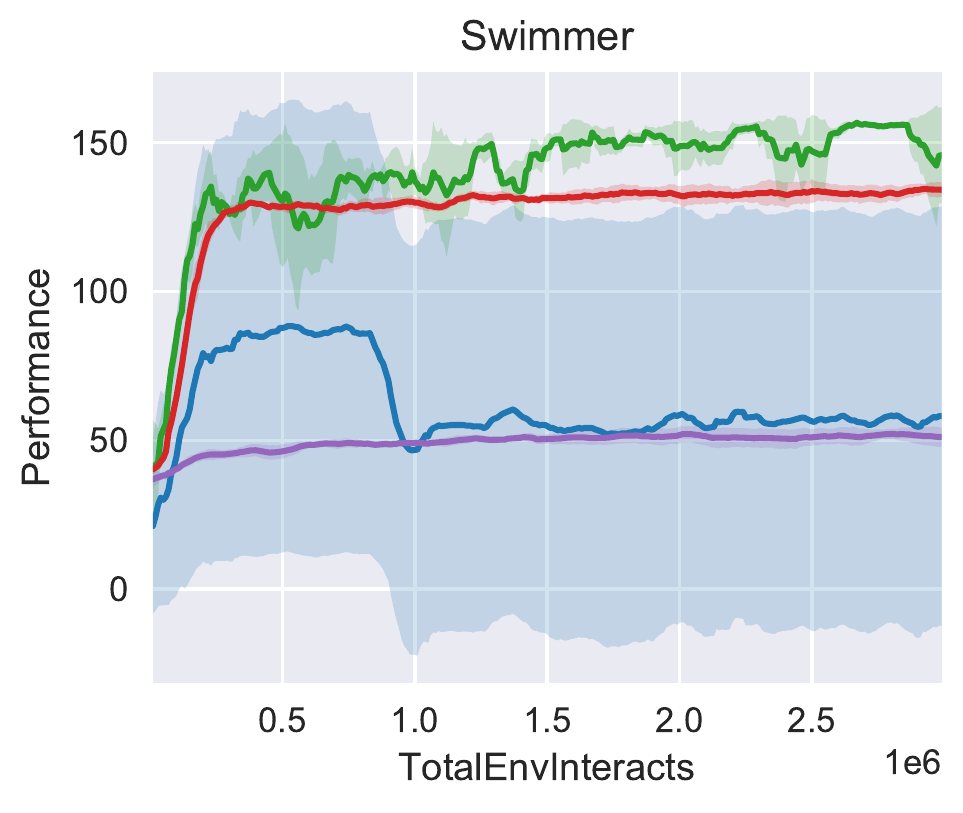}
\end{subfigure}
\begin{subfigure}{}
\includegraphics[width=0.3\textwidth]{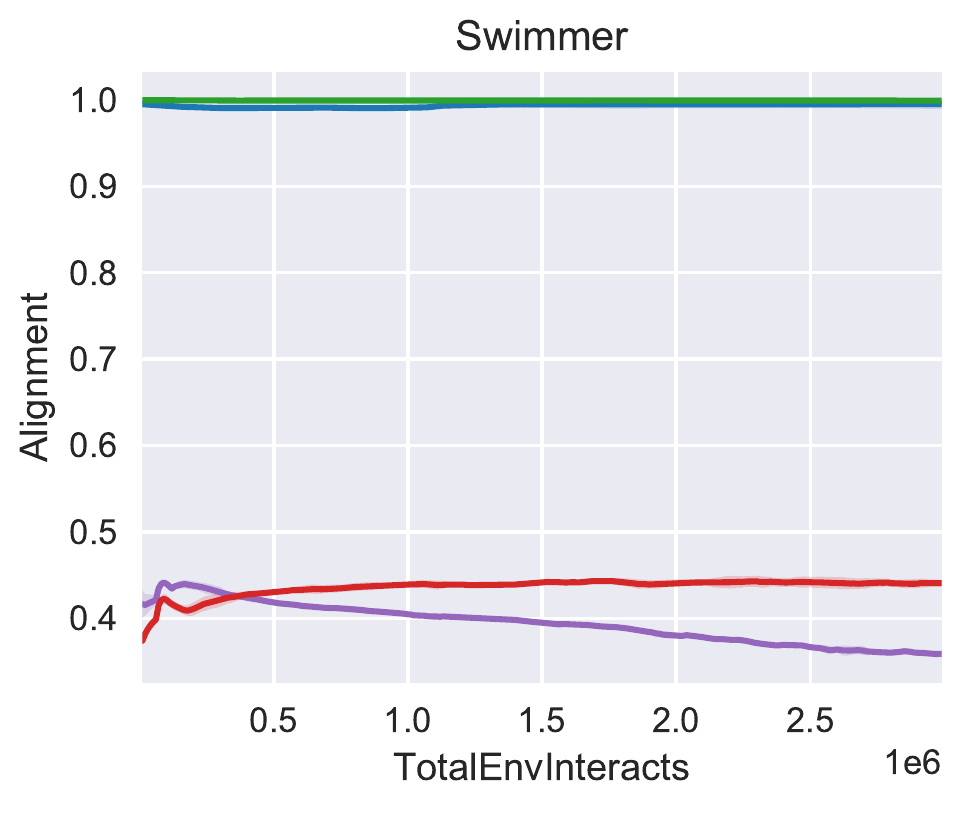}
\end{subfigure}
\begin{subfigure}{}
\includegraphics[width=0.3\textwidth]{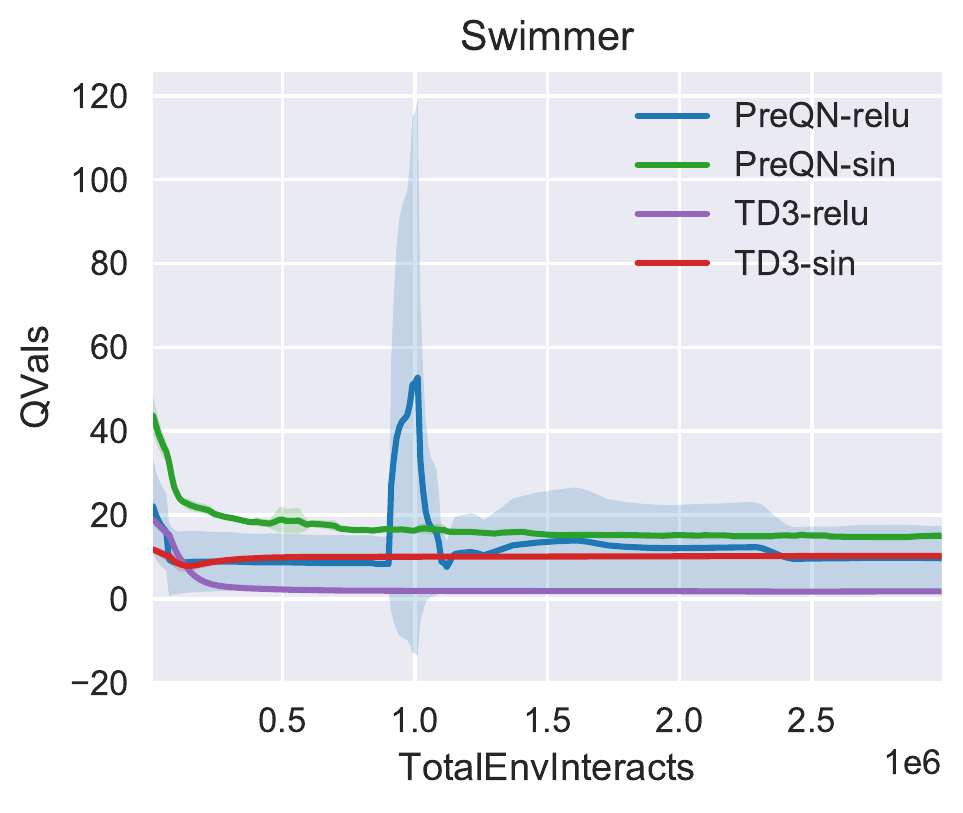}
\end{subfigure}
\caption{Comparison between PreQN and TD3 for relu and sin activation functions in the Swimmer-v2 gym environment. Results averaged over 3 random seeds.}
\end{figure}

\begin{figure}[H]
\centering
\begin{subfigure}{}
\includegraphics[width=0.3\textwidth]{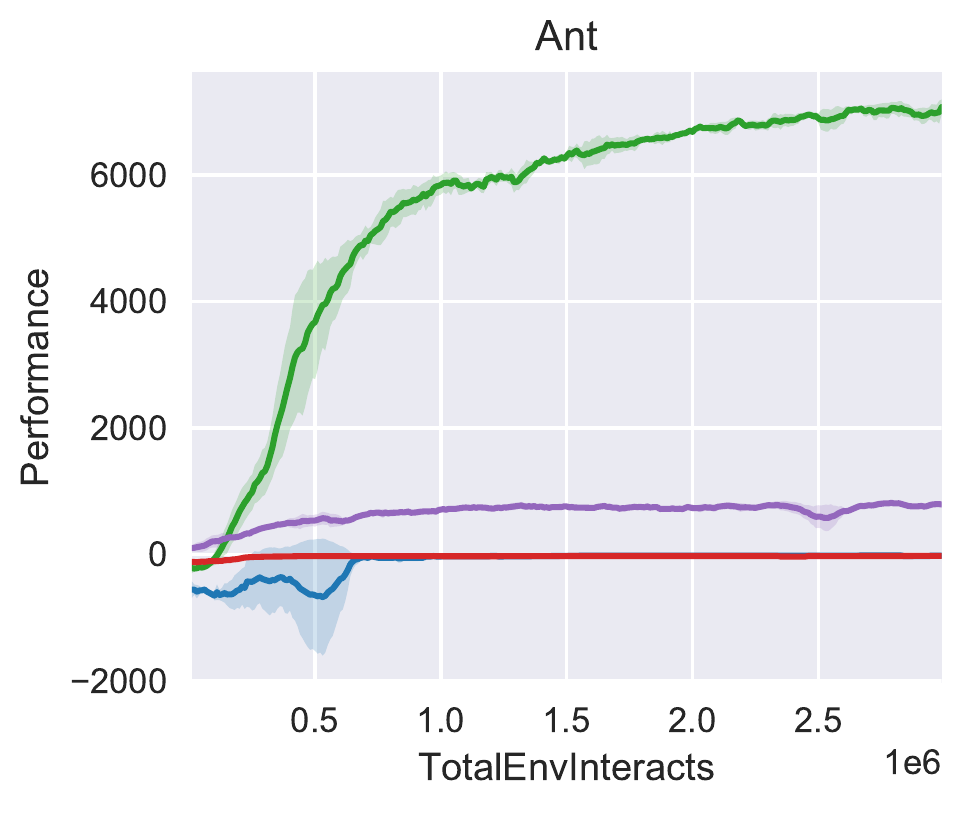}
\end{subfigure}
\begin{subfigure}{}
\includegraphics[width=0.3\textwidth]{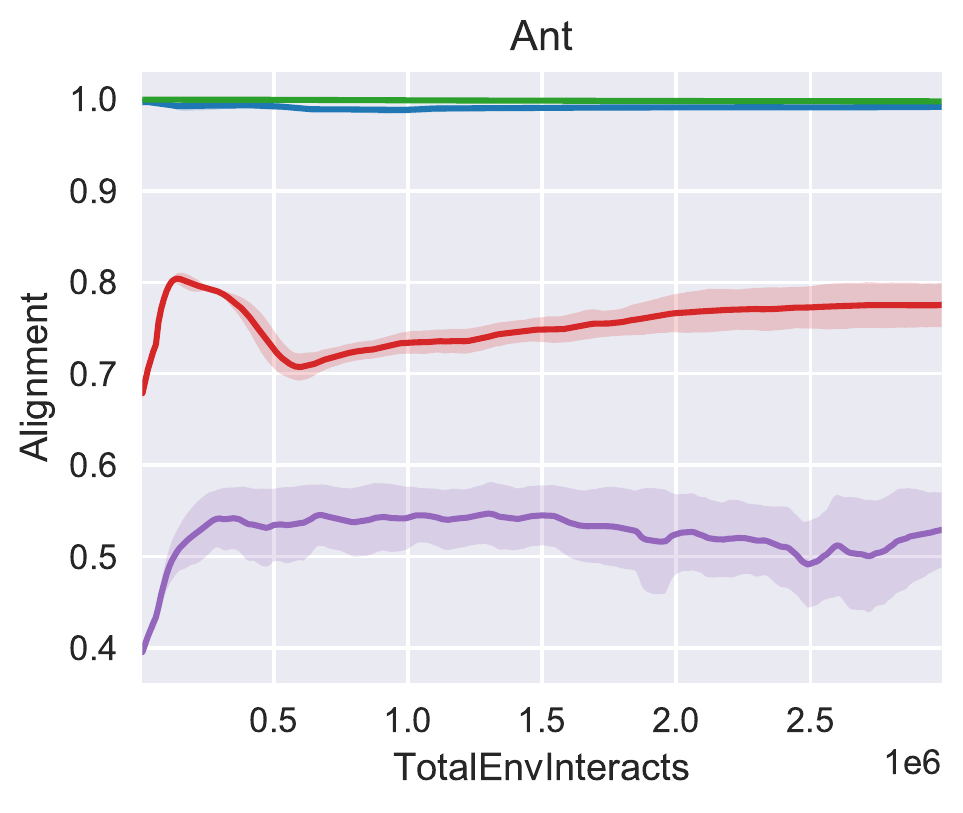}
\end{subfigure}
\begin{subfigure}{}
\includegraphics[width=0.3\textwidth]{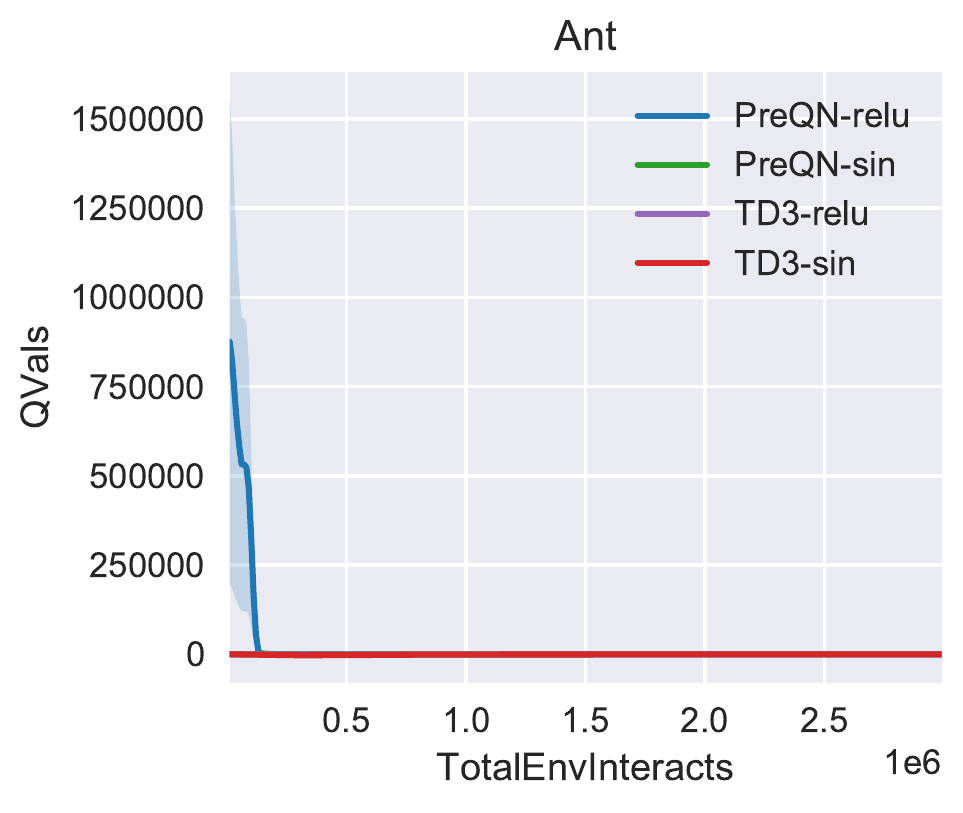}
\end{subfigure}
\caption{Comparison between PreQN and TD3 for relu and sin activation functions in the Ant-v2 gym environment. Results averaged over 3 random seeds.}
\end{figure}

\end{document}